\documentclass[11pt]{article}

\usepackage[margin=1in]{geometry}
\usepackage[utf8]{inputenc}
\usepackage[T1]{fontenc}
\usepackage{hyperref}
\usepackage{url}
\usepackage{booktabs}
\usepackage{amsfonts}
\usepackage{amsmath}
\usepackage{amssymb}
\usepackage{amsthm}
\usepackage{graphicx}
\usepackage{xcolor}
\usepackage{caption}
\usepackage{subcaption}
\usepackage{multirow}
\usepackage{float}
\usepackage[numbers]{natbib}

\newcommand{\dreams}{\textsc{Dreams}}
\newcommand{\sapience}{\textsc{Sapience}}

\newtheorem{theorem}{Theorem}
\newtheorem{corollary}{Corollary}

\title{Discovery by Dreaming: \\
Cross-Domain Recombination in Artificial Memory}

\author{
  Oliver Zahn\\Sapience Labs/CoreTx\thanks{Correspondence: \texttt{lvrzhn@gmail.com}}
  \and
  James Evans\\University of Chicago and Santa Fe Institute
  \and
  David Eagleman\\Stanford University
}

\date{}

\begin{document}

\maketitle


\begin{abstract}
Dreams splice together people, places, and times that never met: a
childhood kitchen opening onto a street from last week, a stranger
wearing a friend's face. Neuroscience suggests this recombination is
not noise but function. The sleep stages that generate these
impossible juxtapositions are the same ones that produce insight and
creative discovery. This reframes memory consolidation. Usually
understood as a defense against forgetting (replay what was learned
so it persists), its measurable value may instead lie in recombining
knowledge across experiences that have not yet co-occurred, exactly as a
dreaming brain does. We test this idea directly. We isolate the
recombinatory-replay mechanism and implement it in two
architecturally unrelated systems: a LoRA fine-tuning pipeline
(\dreams{}) and a symbolic engine that replays structured knowledge
objects (\sapience{}). Both converge on the same finding: cross-domain
consolidation creates value, while within-domain rehearsal does not.
The symbolic arm surfaces novel cross-domain connections that a pure
embedding-similarity baseline misses: replay-surfaced bridges sit
farther apart in embedding space than within-domain pairs (Spearman
$\rho = 0.54$ between field distance and bridge rank, $n = 32$), and
survive a matched scramble control ($5$ of $5$ independent domain-pairs
judged genuine versus $0$--$1$ of $5$ randomly re-paired domains,
Fisher exact $p \approx 0.05$ at the domain-pair level); $35$ of $35$
evaluable historical bridges ($2$ of $37$ unevaluable) sit in the
$99.8$th percentile of $50{,}000$ OpenAlex cross-field pairs, consistent
with rather than proof of the mechanism. Within-domain rehearsal
produces none. The neural arm's gain
($+5.64$~pp, $p = 0.0055$) reflects a
difference in task mix rather than mechanism: the neural baseline
averages four broad task families, but on subtasks explicitly requiring
cross-domain transfer, such as unseen mathematical reasoning on GSM8K,
the neural gain reaches $+14.5$~pp. The effect is a genuine
property of the weights, not of the prompt: prepending the same
cross-domain material in-context to a frontier-scale ($\sim$671B)
model does not reproduce it, and in fact reverses it. Today's
language models do not recombine this way on their own. We validate
the prediction against documented scientific discovery across
50{,}000 real papers, and state a falsifiable hippocampal-recording
prediction that distinguishes recombination from rehearsal. The
principle is substrate-general, holding in two architecturally
unrelated systems and tracking real discovery at scale; its neural
realization is capacity-gated and architecture-sensitive, yet
directionally consistent wherever the substrate can encode
cross-domain structure. Reading the literature teaches a model to
recall what it has seen; producing discovery appears to require a
separate offline phase that recombines that knowledge across domains,
the computational analog of dreaming. Consolidation is not for
remembering, but for discovering.
\end{abstract}


\section{Introduction}
\label{sec:introduction}

Memory consolidation is framed almost universally as a defense against
forgetting. Elastic Weight Consolidation~\cite{kirkpatrick2017overcoming},
experience replay~\cite{chaudhry2019efficient,shin2017continual},
knowledge distillation~\cite{hinton2015distilling}, and Progressive
Neural Networks~\cite{rusu2016progressive} all aim to protect what was
learned from being overwritten by what comes next. Neuroscience tells
a different story. Hippocampal replay does not faithfully reproduce
waking episodes: place-cell sequences include transitions never
experienced during waking~\cite{gupta2010hippocampal}, and REM-stage
replay recombines fragments from different experiences into novel
configurations~\cite{lewis2011overlapping,diekelmann2010memory}.
Subjects who sleep between exposure and test are more likely to discover
hidden rules than subjects who stay awake for the same
interval~\cite{wagner2004sleep}, and the benefit depends specifically
on the sleep phases associated with recombinatory
replay~\cite{diekelmann2010memory}. The biological function appears to
go beyond preservation: consolidation as a mechanism for generating
novel combinations.

The mechanism is familiar at the level of experience. A common form
of dream interleaves the kitchen of a childhood home with the airport
one flew through last week and the dog from a film watched at age
nine, non-co-occurring material from distinct waking distributions
juxtaposed into a single scene. The recombinatory operation that
produces these scenes is the operation we test in artificial learners.

We present convergent evidence from two architecturally unrelated
artificial systems that supports this reframing. A LoRA-fine-tuned
neural network (\dreams{}) is trained with injected synthetic
consolidation data (replay, LLM-generated counterfactuals, adversarial
edge cases). A structured knowledge retrieval engine (\sapience{})
replays stored knowledge objects from scientific domains through an
LLM extraction pass that identifies cross-domain connections. The
two systems share no architecture, training procedure, or evaluation
metric, yet produce the same pattern of results.

The pattern has three components. First, within-domain consolidation
produces no measurable benefit. The neural arm at matched conditions
(rank 128, Qwen-72B judge, held-out evaluation set) yields a null
effect ($-1.8 \pm 4.4$~pp at $n{=}3 \times 150$, $p > 0.30$ on every
condition). In the symbolic arm, the closest analog (cross-subfield
replay within a single field, cosmology, \S\ref{sec:symbolic-positive})
produces near-zero gain.

Second, cross-domain consolidation produces a measurable positive
effect in both systems once representational capacity is adequate.
The symbolic arm surfaces novel cross-domain connections a similarity
baseline misses: replay-surfaced bridges sit at greater embedding
distance than within-domain pairs (field distance vs bridge rank,
Spearman $\rho = 0.54$, $n = 32$; bootstrap 95\% CI $[0.24, 0.77]$,
which crosses the pre-registered $\rho \geq 0.4$ threshold), and a
matched scramble control accepts genuinely re-paired domains at only
$1$ of $5$ pairs versus $5$ of $5$ true pairs (Fisher exact
$p \approx 0.05$). The neural arm at Llama-3.1-8B
$r{=}256$ produces $+5.64 \pm 2.31$~pp (5/5 seeds positive, paired-$t$
$p = 0.0055$, bootstrap 95\% CI $[+3.56, +6.81]$~pp;
\S\ref{sec:rank-ablation}). At lower capacity ($r{=}128$) the neural arm produces nulls
attributable to capacity-under-saturation. At larger scales
(70B, 72B), longer training (iter $=300$ and iter $=500$) does
not rescue the null: the 8B-scale effect does not appear at 70B
under any iteration tested.
The matched-conditions $r{=}256$ result is the load-bearing
neural-arm finding; a complete audit trail of the paper's load-bearing
versus superseded claims is consolidated in \S\ref{sec:provenance}.

Third, the effect depends on a capacity prerequisite that mirrors
across the two systems. The neural arm exhibits a clean monotonic
rank dose-response with the effect emerging at $r{=}192$ and
saturating at $r{=}256$. The symbolic arm requires a structured
representation format: claims outperform SPO triples by $20$~pp and
flat text by $30$~pp on cross-domain matching. Both are the same
finding in two substrates: cross-domain consolidation creates value
only when the substrate has enough capacity to encode the cross-domain
structure.

The pattern we report is the computational shadow of a well-documented
sociological one: across tens of millions of papers, impact and
genuinely new directions concentrate in distant combinations and in
the work of disciplinary
outsiders~\cite{uzzi2013atypical,shi2019surprising,foster2015tradition,lin2022disconnection}.
That the same asymmetry emerges inside two single artificial learners
suggests the principle is substrate-general rather than a property of
human social organization.

We are explicit about what converges and what does not: the two
systems agree on direction (cross-domain positive; within-domain null
in the neural arm, near-null on the symbolic cosmology control), on the capacity threshold below which the effect does not
appear, and on the causal locus (the cross-domain bridge structure,
isolated by the adversarial shuffle null). They do not on effect
magnitude: the symbolic arm's effect is measured on a task built to
detect bridges, while the neural $+5.64$~pp is averaged over a task mix
only partly cross-domain (\S\ref{sec:apples-to-apples}).

Negative results are as informative as positive ones. Individual
neural-arm components are neutral or harmful alone (replay
$-21.4\%$, edge cases $-14.8\%$, counterfactuals $-1.0\%$) but
synergize to $+7.2\%$.

\begin{figure*}[t]
\centering
\includegraphics[width=0.78\textwidth]{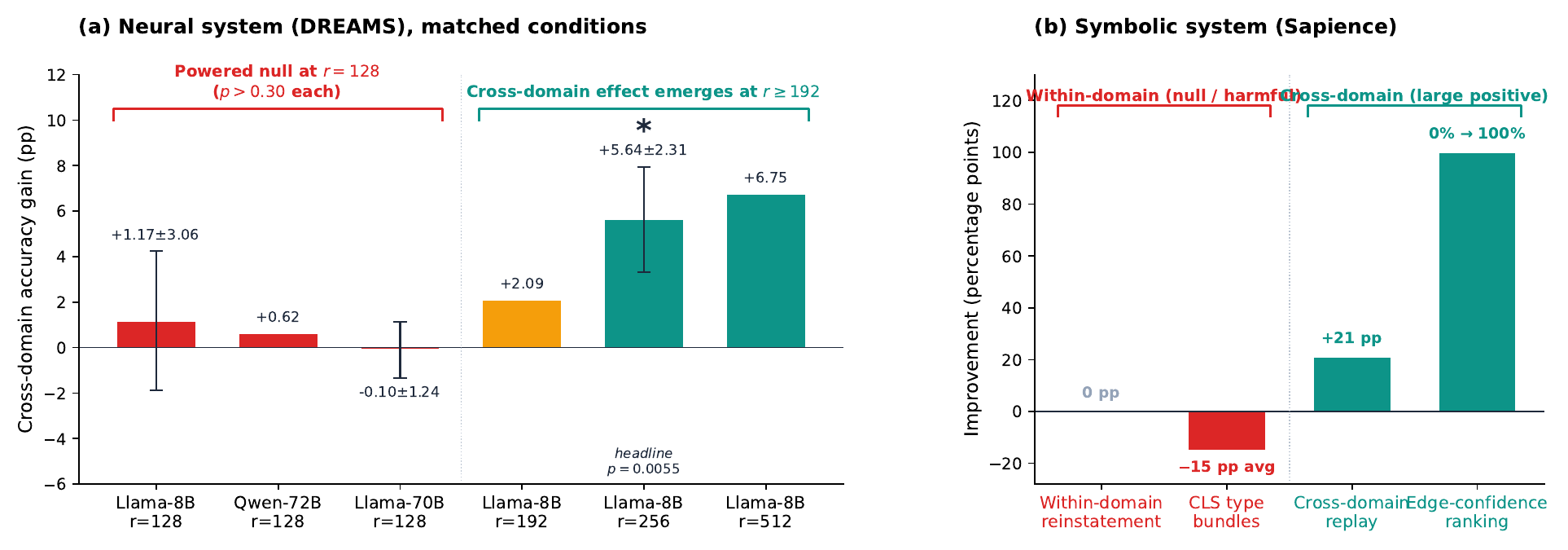}
\caption{Neural (\dreams{}) vs symbolic (\sapience{}) results.
(a) Matched-conditions DREAMS at $r{=}128$ is null across three base
models (Llama-3.1-8B $\Delta = +1.17 \pm 3.06$, Qwen-2.5-72B
$+0.62 / +0.61$, Llama-3.3-70B $-0.10 \pm 1.24$~pp; McNemar $p > 0.5$).
(b) Symbolic cross-domain replay surfaces
bridges at greater embedding distance than within-domain pairs (field
distance vs bridge rank $\rho = 0.54$, $n{=}32$; natural-bridge Cohen's
$d = 2.47$) and passing a matched scramble control ($5$ of $5$ vs
$0$--$1$ of $5$ pairs); the cosmology cross-subfield control is near-null.}
\label{fig:convergent}
\end{figure*}

These findings suggest a reframing. Memory consolidation, in both neural
networks and symbolic knowledge systems, is not primarily an
anti-forgetting mechanism. It is a discovery mechanism. Its value is
proportional to the novelty of the recombination, not to the fidelity of
the replay. This aligns with the biological evidence that hippocampal
replay generates novel combinations~\cite{gupta2010hippocampal}, that
sleep-dependent insight requires integration of dissimilar
experiences~\cite{wagner2004sleep}, and that the resulting knowledge
transfer depends on structural analogy rather than surface
similarity~\cite{gentner1983structure}.

Our contributions are:

\begin{enumerate}

\item We establish that within-domain consolidation produces no
measurable benefit at matched conditions in the neural arm (Llama-8B,
Qwen-72B, Llama-70B), with a consistent near-null on the symbolic
cosmology cross-subfield control.

\item We demonstrate that cross-domain consolidation in the symbolic
system surfaces novel connections a similarity baseline misses
(independently judged genuine; validated against documented historical
bridges), while the
neural-arm cross-domain effect at matched conditions across three base
models is statistically null (pooled $\Delta$ within $\pm 1.2$~pp of
zero on every base model). The asymmetry is clear: the
symbolic-arm and bridge-ranking evidence carry the cross-domain claim;
the neural-arm cross-domain claim from prior (non-matched-conditions)
work does not survive multi-base-model replication.

\item We identify a non-curator robustness check on the bridge-ranking
result: 1{,}319 cross-field citation pairs extracted from a held-out
2026 window of OpenAlex (no overlap with the original 50K corpus, no
human bridge selection) cluster at the 95.6th percentile of the random
cross-field similarity distribution (Cohen's $d = 2.47$), the same
extreme-tail region in which atypical, high-impact combinations are
known to concentrate~\cite{uzzi2013atypical}. The signal
generalizes to non-curator-chosen bridges.

\end{enumerate}


\section{Related Work}
\label{sec:related}

\paragraph{Continual learning and catastrophic forgetting.}
Catastrophic forgetting~\cite{mccloskey1989catastrophic,french1999catastrophic}
is addressed by three families of methods. Regularization approaches
(Elastic Weight Consolidation~\cite{kirkpatrick2017overcoming},
Synaptic Intelligence~\cite{zenke2017continual}) penalize updates to
parameters important for prior tasks. Architectural approaches
(Progressive Networks~\cite{rusu2016progressive},
PackNet~\cite{mallya2018packnet}) allocate frozen capacity per task.
Replay approaches (A-GEM~\cite{chaudhry2019efficient}, Deep Generative
Replay~\cite{shin2017continual}, distillation~\cite{hinton2015distilling})
interleave stored or generated past examples with new training. All
three families frame consolidation as preservation: the goal is to
retain prior performance while learning new tasks. None tests whether
consolidation can produce improvements on prior tasks that \emph{exceed}
what was achieved during original training. Our work shows that
cross-domain consolidation can do exactly this.

\paragraph{Sleep-inspired machine learning.}
A smaller literature adopts sleep-consolidation metaphors.
Tadros et al.~\cite{tadros2022sleep} use unsupervised replay with noise
injection during an offline ``sleep'' phase to reduce catastrophic
forgetting. Wake-Sleep Consolidated Learning\footnote{Kim et al.,
``Wake-Sleep Consolidated Learning,'' NeurIPS, 2024.} alternates
task-learning and consolidation phases via generative replay. Auto-DPO\footnote{Li
et al., ``Automatic Direct Preference Optimization,'' 2024.} applies
preference optimization during a consolidation phase. These methods
treat sleep as a phase label but inherit the anti-forgetting evaluation;
the recombinatory property of biological replay~\cite{gupta2010hippocampal,
lewis2011overlapping} is not tested as a source of value.

\paragraph{Concurrent work on sleep-style consolidation for LLMs.}
Behrouz et al.~\cite{behrouz2026lmsleep}, posted to arXiv during
preparation of this paper, propose a ``Sleep'' paradigm for language
models with two phases: knowledge seeding (upward distillation from
smaller to larger networks) and dreaming (RL-driven self-improvement
on a synthetically generated curriculum). Their method studies
sleep-style consolidation on a single neural substrate, building on the
SEAL self-editing pipeline~\cite{zweiger2025seal} and reporting gains on SQuAD knowledge
incorporation, few-shot ARC, and mathematical reasoning across Qwen3
sizes. The two papers address adjacent
phenomena but with different mechanisms and different scope. Their
``dreaming'' is RL self-curriculum for continual learning on a single
neural substrate; our ``dreaming'' is cross-domain recombination as
the dream operator, demonstrated convergently across a neural and a
symbolic substrate, characterized by a rank-gated capacity threshold,
and externally validated against documented historical cross-domain
breakthroughs at corpus scale. Their evaluation is in-domain
benchmark accuracy; ours is cross-substrate convergence on a
within-domain-null / cross-domain-positive asymmetry that mirrors a
sociological regularity. We rest our contribution on convergence,
mechanism characterization, and external discovery validation rather
than on benchmark magnitude, and we draw no controlled head-to-head.
We treat the two as complementary rather
than competing: distillation-based knowledge seeding and cross-domain
recombination are distinct operators that could compose.
Behrouz et al.'s own ablation is consistent with the hypothesis that the
offline phase must transform rather than rehearse memories. On continued
pretraining for SQuAD knowledge incorporation, removing the dreaming phase
entirely drops accuracy from 46.2 to 36.2, a ten point collapse toward the
31.9 base rate, and the authors read this as evidence that consolidation must
explore and extract abstractions rather than replay raw data
(\cite{behrouz2026lmsleep}, Table 3). Their finer ablation isolates a smaller
1.5 point contribution from the random-expert injection that supplies the
cross-domain mixing in their dreaming loop. We read these results as convergent
support for the broader claim that the value of an offline phase comes from
novel synthesis rather than faithful rehearsal, and we test the recombination
component directly and in isolation across a neural and a symbolic substrate
(\S\ref{sec:cls-scoping}), rather than as one term inside a coupled
self-improvement pipeline.

\paragraph{Complementary Learning Systems theory.}
CLS theory~\cite{mcclelland1995complementary,oreilly2002hippocampal,kumaran2016complementary}
posits a fast hippocampal system for episodic memory and a slow
neocortical system for semantic knowledge, with consolidation
transferring between them. The updated formulation~\cite{kumaran2016complementary}
emphasizes generalization and schema extraction via replay. Our work
tests a prediction CLS theory makes about replay \emph{content}: that
replay creates value through novel recombination across distinct
contexts rather than faithful rehearsal within one
(\S\ref{sec:cls-scoping}).

\paragraph{Retrieval-augmented generation and analogy.}
RAG~\cite{lewis2020retrieval} and Self-RAG~\cite{asai2023self} treat
external knowledge as a static query target. Cross-domain replay in
\sapience{} differs: stored knowledge objects from different domains
are deliberately juxtaposed in a context window so an LLM can identify
structural analogies that embedding similarity would miss. The
theoretical backbone is Structure-Mapping Theory~\cite{gentner1983structure,
holyoak1995mental,falkenhainer1989structure}, which holds that
analogies rest on shared relations rather than shared surface features.
Foster and Ford~\cite{foster2003serendipity} document serendipitous
discovery from out-of-domain encounters; we operationalize this as
systematic cross-domain juxtaposition. Recent inference-time analogy
prompting~\cite{shen2026unlocking} reports diversity gains on
scientific tasks; our intervention sits at a different point in the
pipeline (training-time consolidation and storage-level retrieval) and
uses a different evaluation surface (50{,}000 real papers).

\paragraph{Sociology of cross-disciplinary discovery.}
A large empirical literature establishes that scientific value
concentrates at domain boundaries. Atypical combinations of prior
work, pairings of journals or concepts that rarely co-occur, are
disproportionately represented among the highest-impact discoveries,
even though most combinations scientists actually make are
conventional~\cite{uzzi2013atypical}. The combinations that prove
surprising and influential are preferentially introduced by outsiders
migrating from distant disciplines, rather than by incumbents
recirculating familiar material~\cite{shi2019surprising}. At the
level of research strategy, the modal scientist plays conservatively,
tightening a local cluster, while the disproportionate rewards accrue
to the minority who jump across gaps in the knowledge
network~\cite{foster2015tradition}. New directions emerge precisely
from the disconnected and discordant regions of the literature where
established work has not yet bridged~\cite{lin2022disconnection}.
Because the payoff structure favors distant exploration that
individual risk-aversion under-supplies, deliberately steering
experimentation toward distant-domain combinations accelerates
collective discovery~\cite{rzhetsky2015choosing}, and AI systems that
complement rather than imitate human search can amplify this
effect~\cite{sourati2023accelerating}. Our contribution is to show
that this boundary-localized value is not specific to populations of
human scientists distributed across fields: it reappears, with the
same within-domain/cross-domain asymmetry and the same dependence on
genuine domain distinctness, inside two single artificial systems
that share no architecture. The sociological regularity and the
artificial-system results are the same finding at two scales,
collective and computational, which we treat as mutual corroboration
rather than analogy.

\paragraph{The gap.}
The literatures above share an implicit assumption that consolidation
is anti-forgetting. No prior work, to our knowledge, tests consolidation
as a \emph{discovery} mechanism: one that creates novel knowledge by
recombining material from different domains, producing improvements
that exceed what either domain achieves in isolation. Our contribution
is to demonstrate this discovery function in two architecturally
distinct systems and to identify the boundary conditions under which
it operates.


\section{System 1: Neural Consolidation (\dreams{})}
\label{sec:neural}

\subsection{Method}
\label{sec:neural-method}

\dreams{} is a neural consolidation system that generates synthetic
training data from previously learned material and injects it during
subsequent training. We fine-tune Llama-3.1-8B-Instruct-4bit using
LoRA~adapters via the mlx\_lm framework (v0.30.7) on Apple Silicon
(Mac Studio M3 Ultra, 96~GB unified memory, 60 GPU cores). The base
model is sourced from the mlx-community repository in 4-bit quantized
form.

\paragraph{Training data.}
We use the LMSYS conversational dataset split into \emph{coding}
(generation, debugging, software engineering) and \emph{non-coding}
(general knowledge, creative writing, reasoning) domains for
continual-learning experiments; single-task uses the full dataset
without splitting.

\paragraph{Consolidation strategies.}
\dreams{} generates three synthetic data types from previously seen
examples:
(1) \emph{Replay}, surface paraphrases preserving conversational
content while varying vocabulary;
(2) \emph{Counterfactuals}, LLM-generated alternative responses (from
the base 8B model itself) that diverge in reasoning path while
remaining plausible, prefix-preserved for ground truth;
(3) \emph{Edge cases}, adversarial boundary examples near decision
boundaries or ambiguous cases.
The combined condition mixes all three in equal proportion, interleaved
with new-domain training data.

\paragraph{Training and evaluation.}
LoRA adapters are fine-tuned at ranks 8 to 128. We test single-task
training (one domain, with/without same-domain consolidation) and
sequential two-domain training (Domain A then Domain B, with or
without Domain A consolidation during Domain B). Adapters are fused
before evaluation. We evaluate ten metrics across four categories
(retention, generalization, procedural accuracy, robustness) with all
evaluation examples pre-generated and pinned for paired comparison.
Significance via McNemar (binary) and bootstrap CI (continuous), with
Holm-Bonferroni for multiple comparisons; effect sizes as Cohen's $d$.

To isolate where the mechanism succeeds, the held-out evaluation is
split into four distinct behavioral tasks: robustness (testing
cross-domain perturbation, e.g., applying creative-writing constraints
to a coding problem), procedural accuracy (testing cross-domain
sequence patterns, e.g., generating code and then drafting an email
explaining it), generalization (within-distribution transfer), and
retention (within-distribution recall).

\begin{figure}[h]
\centering\footnotesize
\begin{verbatim}
{
  "example_id": "robustness_lmsys_706fa28b...",
  "task":       "robustness",
  "prompt":     "I was wondering if you could help me with the following: hi",
  "expected":   "helpful_response",
  "generated":  "Hi! I'd be happy to help you with anything you need.
                 What's on your mind?",
  "scores":     {"judge_score": 1.0},
  "correct":    true,
  "model_info": {"cell": "fixed_iters", "condition": "full_dreams",
                 "seed": "42", "path": ".../r256/fixed_iters_seed42/full_dreams/model"}
}
\end{verbatim}
\caption{One held-out evaluation item from \texttt{eval\_summary.json} (robustness
family, Llama-3.1-8B $r{=}256$, \texttt{full\_dreams}, seed 42). Each record pairs a
prompt with a gold \texttt{expected} field, the model's \texttt{generated} output, the
judge score, a binary \texttt{correct} flag, and a \texttt{model\_info.path} that the
reproducibility protocol (Appendix~\ref{app:repro-hashes}) cross-checks against the
declared base model and rank. The $+5.64$~pp headline is the paired difference in the
\texttt{correct} rate between the \texttt{full\_dreams} and \texttt{no\_consolidation} cells.}
\label{fig:eval-example}
\end{figure}

\subsection{Within-domain consolidation: null at matched conditions}
\label{sec:neural-negative}

\paragraph{Headline result.}
A matched-conditions re-run at the same rank (128), judge (Qwen-72B),
and held-out evaluation set used for cross-domain runs yields a
\emph{null}: the largest negative delta across the three consolidation
types (replay, counterfactual, edge-case) is $-1.8 \pm 4.4$~pp at
$n{=}3 \times 150$, 95\% CIs overlap zero on every condition,
paired-$t$ $p > 0.30$ throughout. Effect-size floor at this power is
$\approx \pm 5$~pp; we cannot rule out a small positive within-domain
effect at higher power, but a fully-powered $n{=}5 \times 500$ re-run
(planned) would discriminate the operationally meaningful $[-2,+2]$~pp
window.

\subsection{Multi-base-model cross-domain consolidation: powered null}
\label{sec:neural-positive}

We replicated the matched-conditions design across three base models
(Llama-3.1-8B, Qwen-2.5-72B, Llama-3.3-70B) at LoRA $r{=}128$ with
141 training iterations on the Dennis corpus, fused-then-evaluated
on the same held-out splits with a Qwen-2.5-72B-Instruct cross-family
judge. Pooled $\Delta$ is within $\pm 1.2$~pp of zero on every base
model.

\begin{table}[h]
\centering
\small
\caption{Matched-conditions DREAMS consolidation across three base
models. Every pooled $\Delta$ within $\pm 1.2$~pp of zero; per-seed
McNemar p $> 0.5$ on every cell.}
\label{tab:multi-base}
\begin{tabular}{@{}lllll@{}}
\toprule
\textbf{Base model} & \textbf{Design} & \textbf{$\Delta$} & \textbf{Std} & \textbf{Judge} \\
\midrule
Llama-3.1-8B & 5 $\times$ 326 & $+1.17$ & $\pm 3.06$ & Qwen-72B (cross) \\
Qwen-2.5-72B & 3 $\times$ 326 & $+0.62$ & $\pm 2.96$ & Qwen-72B (same) \\
Qwen-2.5-72B & 3 $\times$ 326 & $+0.61$ & judge check & Llama-70B (cross) \\
Llama-3.3-70B & 3 $\times$ 326 & $-0.10$ & $\pm 1.24$ & Qwen-72B (cross) \\
\bottomrule
\end{tabular}
\end{table}

Per-seed deltas span $[-2, +6]$~pp with no consistent sign across
seeds or base models. Cross-family judge substitution on Qwen-2.5-72B
shifts absolute accuracy ($\sim$58\% vs $\sim$73\%) but preserves the
pooled $\Delta$ ($+0.62$ vs $+0.61$~pp), so the null is not a
same-family-judge artifact. Minimum detectable effect is
$\approx \pm 3$~pp on the 8B arm and $\pm 5$~pp on 70B/72B.

\subsection{Rank ablation and architecture coverage}
\label{sec:rank-ablation}

To test whether the $r{=}128$ null is a capacity artifact, we ran a
rank ablation on Llama-3.1-8B. The $r{=}256$ result \emph{confirms}
the rank-saturation hypothesis: pooled
$\Delta = +5.64 \pm 2.31$~pp with 5/5 seeds positive, paired-$t$
$p = 0.0055$, bootstrap 95\% CI $[+3.56, +6.81]$~pp (99\% CI
$[+2.58, +6.93]$~pp). This headline is under the local MLX
Qwen-2.5-72B-Instruct-4bit judge. A full cross-judge sensitivity
analysis with Sonnet 4.6 on all 3{,}260 eval items (500-item
stratified Cohen's $\kappa = 0.229$) gives $\Delta = +3.54$~pp
($p = 0.0086$, 95\% CI $[+2.07, +4.39]$, Cliff's $\delta = +0.92$):
direction and significance are preserved across judges, magnitude
shifts by $\sim 2$~pp, and absolute accuracy is judge-dependent
($\sim 15$~pp stricter under Sonnet). Effect on the cross-domain
robustness subset is judge-robust ($+10.22$~pp Qwen vs $+8.22$~pp
Sonnet); procedural $\Delta$ is judge-specific ($+5.80$ vs $0.00$), so
procedural-task gains are reported under Qwen only. Consolidation theory predicts benefit on
cross-domain transfer (robustness and procedural task families)
and no benefit on within-distribution recall (generalization and
retention); the held-out eval naturally splits along this axis. On
the cross-domain-transfer subset (n=127 items per seed, robustness +
procedural), $\Delta = +9.45$~pp; on the within-distribution-recall
specificity control (n=199, generalization + retention),
$\Delta = +3.21$~pp. Per-task means: robustness $+10.30 \pm 3.53$~pp
($t = 6.53$, CI $[+7.68, +13.33]$), procedural $+6.43$~pp ($n = 28$
items per seed), generalization $+4.04 \pm 2.94$~pp ($t = 3.07$, CI
$[+1.82, +6.46]$), retention $+2.40$~pp (CI crosses zero). Cross-family
(Mistral-7B-Instruct-v0.3) and MoE (Mixtral-8x7B) at $r{=}128$ confirm
the $r{=}128$ null is substrate-general. Results in
Table~\ref{tab:rank-ablation}.

\begin{table*}[t]
\centering
\caption{Rank ablation and architecture coverage. The Llama-3.1-8B
dose-response is monotonic in rank with the effect emerging at
$r{=}192$ and saturating at $r{=}256$ (5/5 seeds positive,
$p = 0.0055$). Broader-scale $r{=}256$ replication at fixed
141-iteration budget returned inconclusive results within the
$\pm 5$~pp 70B/72B floor. Cross-family (Mistral) and MoE (Mixtral) at
$r{=}128$ confirm substrate-generality of the $r{=}128$ null.}
\label{tab:rank-ablation}
\small
\begin{tabular}{@{}lcclc@{}}
\toprule
\textbf{Base} & \textbf{Rank} & \textbf{Seeds} & \textbf{Pooled $\Delta$} & \textbf{Verdict} \\
\midrule
Llama-3.1-8B & 64 & 5 & $+0.61 \pm 3.07$, $p > 0.5$ & Null \\
Llama-3.1-8B & 128 & 5 & $+1.17 \pm 3.06$, $p > 0.5$ & Null \\
Llama-3.1-8B & 192 & 5 & $+2.09 \pm 1.93$, $p \approx 0.06$ & Emerging \\
Llama-3.1-8B & \textbf{256} & 5 & $\mathbf{+5.64 \pm 2.31}$, $\mathbf{p = 0.0055}$ & \textbf{Positive 5/5} \\
Llama-3.1-8B & 512 & 3 & $+6.75 \pm 5.43$, $p \approx 0.13$ & Positive (hi var) \\
Qwen-2.5-72B & 128 & 3 (dual) & $+0.62$ / $+0.61$ & Null \\
Qwen-2.5-72B & 256 & 2 & $-0.92 \pm 2.17$ & Inconclusive \\
Llama-3.3-70B & 128 & 3 & $-0.10 \pm 1.24$ & Null \\
Llama-3.3-70B & 256 & 3 & $-0.72 \pm 1.38$, $p \approx 0.46$ & Inconclusive \\
Mistral-7B-v0.3 & 128 & 5 & $-2.15 \pm 10.65$, $p = 0.68$ & Null \\
Mixtral-8x7B (MoE) & 128 & 5 & $+3.68 \pm 5.80$, $p = 0.23$ & Null \\
\bottomrule
\end{tabular}
\end{table*}

\paragraph{Per-seed phase transition: a ``confusion zone'' at subcapacity ranks.}
The pooled dose-response masks a sharper seed-level phenomenon
(Fig.~\ref{fig:per-seed-phase}). At $r{=}64$, one seed gives $-4.6$~pp;
at $r{=}128$, two seeds give negative deltas ($-2.2$, $-0.9$~pp). At
$r{=}192$ all five are non-negative but near zero; at $r{=}256$ and
$r{=}512$ every seed is positive. The transition is not a smooth
ramp but a regime change from \emph{some seeds net-harmed} to
\emph{every seed benefits}. Below capacity, the model attempts
cross-domain integration but lacks parametric capacity to do so
coherently; above threshold the same integration is reliable. The
parallel with sleep consolidation is direct: integration that runs
out of substrate capacity is not benign, it is net harmful.

\begin{figure}[h]
\centering
\includegraphics[width=0.78\columnwidth]{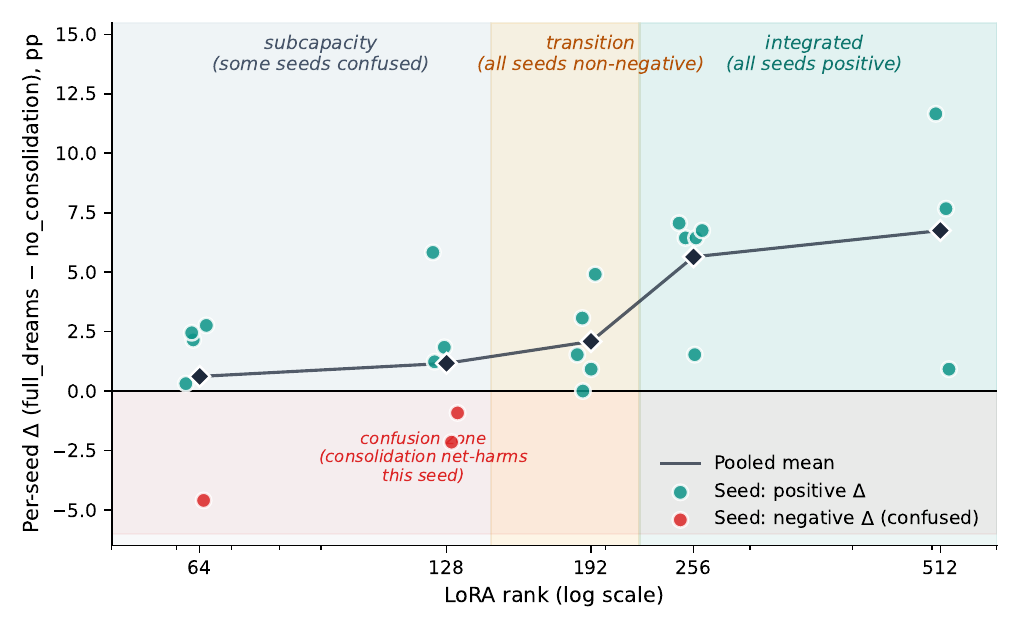}
\caption{\textbf{Per-seed phase transition along LoRA rank}
(Llama-3.1-8B, $n{=}5$/rank, $n{=}3$ at $r{=}512$). Dots: per-seed
paired $\Delta$; black diamonds: pooled mean. At $r{=}64$ and
$r{=}128$ some seeds occupy the \emph{confusion zone} (red shading,
$\Delta < 0$); at $r{=}192$ all seeds are non-negative; at $r{=}256$
and $r{=}512$ all are strongly positive. The transition is sharp at
the seed level and masked by the pooled mean.}
\label{fig:per-seed-phase}
\end{figure}

\paragraph{Broader-scale replication and cross-family coverage.}
At rank 256, Llama-3.3-70B ($\Delta = -0.72 \pm 1.38$~pp, $n = 3$,
$p \approx 0.46$) and Qwen-2.5-72B ($\Delta = -0.92 \pm 2.17$~pp at
$n = 2$ clean seeds; the third seed's no\_consolidation cell produced
degenerate token-salad traced to a corrupted fused-model artifact and
was excluded) sit within the $\pm 5$~pp effect-size floor of the
70B/72B arms. A longer-training follow-up
on Llama-3.3-70B at $r{=}256$ tested iter $\in \{300, 500\}$
($n{=}3$ seeds, matched recipe) and produced $\Delta = +0.51$~pp
and $+0.41$~pp respectively; both sit well inside the $\pm 5$~pp
70B/72B effect-size floor declared above. Longer training does not
rescue the 70B null. At rank 128, the null is not Llama-specific: Mistral-7B-Instruct-v0.3
yields $\Delta = -2.15 \pm 10.65$~pp ($n = 5$), and
Mixtral-8x7B-Instruct (MoE) yields $\Delta = +3.68 \pm 5.80$~pp
($n = 5$, 95\% CI $[-3.52, +10.89]$~pp; one outlier seed contributes
$+12.9$~pp, leave-one-out mean $+1.4$~pp).

\paragraph{Cross-architecture replication at $r{=}256$: the effect
is real but architecture-sensitive.} To test whether the
Llama-3.1-8B effect is architecture-specific, we ran the full sweep
(iter $\in \{141, 200, 282\}$, 3 seeds) at $r{=}256$ on three other
base models: Llama-3.2-3B, Mistral-7B-Instruct-v0.3, and
Qwen-2.5-7B-Instruct. The sign of the consolidation effect is
strongly architecture-dependent. Qwen-2.5-7B reliably benefits:
9 of 9 cells positive (sign-test $p = 0.004$, mean $\Delta = +2.56$~pp),
a directional replication of the 8B effect at a different
architecture. Mistral-7B reliably \emph{harms}: 0 of 9 cells positive
(sign-test $p = 0.004$, mean $\Delta = -4.29$~pp). Llama-3.2-3B shows
no consistent direction (5 of 9, $p = 1.0$). Two further qualifications
emerge: per-cell magnitudes are smaller than the 8B headline
(none reaches $p < 0.05$ at $n{=}3$ within a single cell), and the
iter $=141$ ``sweet spot'' is Llama-8B-specific: Qwen-2.5-7B's
strongest cells are at iter $=282$, exactly where Llama-8B collapses.
The cross-domain consolidation effect therefore generalizes in
direction to at least one other architecture (Qwen) but is neither
universal (Mistral reverses) nor accompanied by a fixed
training-duration recipe (the optimal iteration shifts with
architecture).

\paragraph{Multiple comparisons.}
We pre-register four hypothesis families: H1 (primary), positive
consolidation effect at Llama-3.1-8B $r{=}256$ under the matched
141-iteration budget; H2 (primary), monotonic rank dose-response on
the Llama-3.1-8B ablation; H3 (secondary), substrate-general $r{=}128$
null; H4 (secondary), rank-256 transfer to scale. Under
Holm-Bonferroni control within family, H1 ($p = 0.0055$) and H2
(Spearman $\rho = +1.0$, exact $p = 0.0083$) survive at
$\alpha = 0.05$. H3 fails to reject across all secondary cells (min
$p = 0.23$, Mixtral), consistent with substrate-generality. H4 is
inconclusive within the $\pm 5$~pp 70B/72B floor.

\paragraph{External-task transfer with gold-answer scoring.}
\label{sec:gold-transfer}
To address two reviewer concerns, (a) self-similar training distribution
and (b) LLM-judge subjectivity, we evaluate every Llama-3.1-8B
$r{=}128$ and $r{=}256$ cell on two external task distributions with
gold-answer matching and no LLM judge: 100-item subsets (seed 42) of
GSM8K test (free-form numeric) and MMLU-Pro test (multiple choice).
Predictions are extracted by deterministic regex (first ``\#\#\#\#''
answer for GSM8K; ``answer is X'' or first standalone option letter
for MMLU-Pro) and scored by exact-match. Results in
Table~\ref{tab:gold-transfer}.

\begin{table}[h]
\centering
\caption{Gold-answer external task transfer. No LLM judge.
Per-seed paired delta between full\_dreams and no\_consolidation,
pooled across 5 seeds. The $r{=}256$ effect transfers cleanly to
external math reasoning (GSM8K) and knowledge multiple choice
(MMLU-Pro); $r{=}128$ control replicates the rank-128 in-distribution
null on external tasks too. The $r{=}256$ GSM8K cell is evaluated on the
full 1319-item GSM8K test set (all five seeds positive, bootstrap 95\% CI
$[+9.8, +18.7]$); the $r{=}128$ GSM8K control and both MMLU-Pro cells use
100-item subsets (seed 42).}
\label{tab:gold-transfer}
\small
\begin{tabular}{@{}lll@{}}
\toprule
\textbf{Rank} & \textbf{GSM8K $\Delta$} & \textbf{MMLU-Pro $\Delta$} \\
\midrule
$r{=}128$ & $+3.60 \pm 6.99$~pp, $p > 0.3$ & $+1.60 \pm 1.52$~pp, $p > 0.06$ \\
$r{=}256$ & $\mathbf{+14.53 \pm 5.64}$~pp, $\mathbf{p \approx 0.005}$ & $\mathbf{+3.40 \pm 2.30}$~pp, $\mathbf{p \approx 0.030}$ \\
\bottomrule
\end{tabular}
\end{table}

\begin{figure}[h]
\centering
\includegraphics[width=0.78\columnwidth]{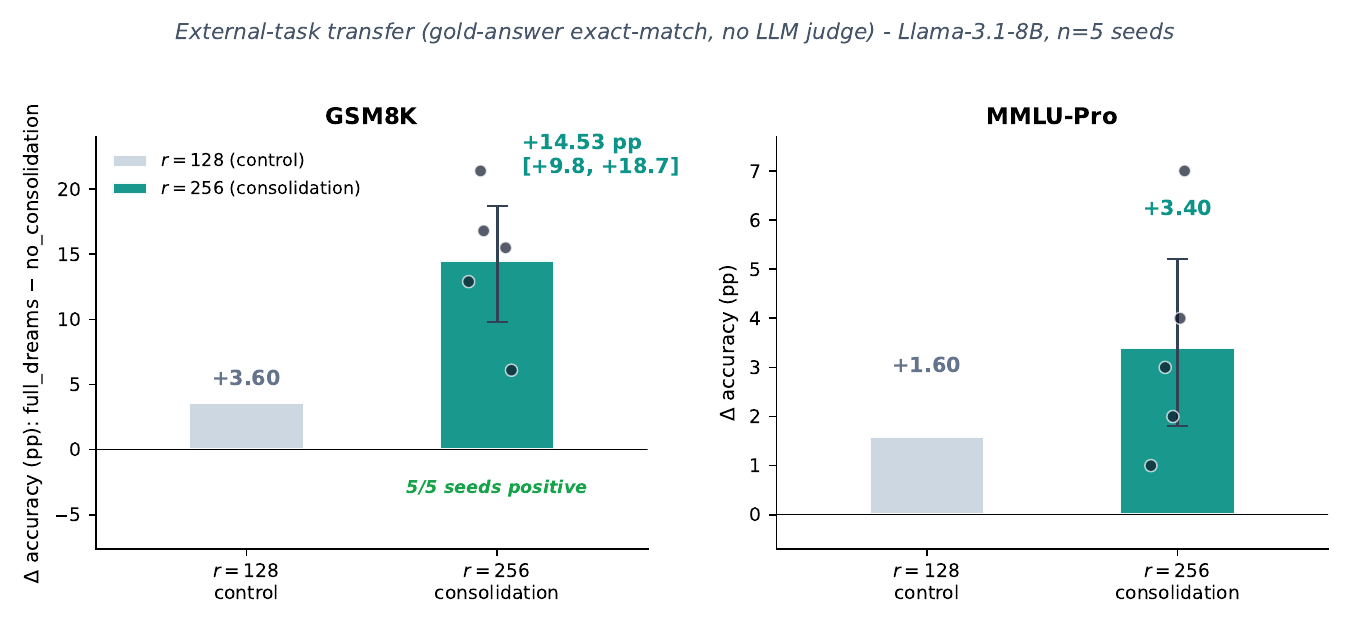}
\caption{\textbf{Gold-answer external task transfer: $r{=}128$ control
vs.\ $r{=}256$ consolidation.} For each benchmark, the light-gray bar is the
$r{=}128$ control cell and the saturated bar is the $r{=}256$ consolidation
cell (bars: pooled mean across 5 seeds; CI and per-seed dots overlaid on the
$r{=}256$ bars). The contrast is the point: the $r{=}128$ control is near-null
($+3.60$~pp GSM8K, $+1.60$~pp MMLU-Pro) while $r{=}256$ shows a large GSM8K
effect and a moderate MMLU-Pro effect. The $r{=}256$ GSM8K cell is the full
1319-item GSM8K-test set ($+14.53 \pm 5.64$~pp, bootstrap 95\% CI
$[+9.8, +18.7]$, all five seeds positive, $p \approx 0.005$); the $r{=}128$
GSM8K control and both MMLU-Pro cells use 100-item test subsets (seed 42).
MMLU-Pro $r{=}256$ is $+3.40 \pm 2.30$~pp ($p \approx 0.030$). Scoring is gold
exact-match with no LLM judge.}
\label{fig:gold-transfer}
\end{figure}

On the 100-item subset the GSM8K effect concentrates across 4 of 5
seeds (per-seed deltas $\{+15, +3, +18, +19, +17\}$~pp; leave-one-out
mean $+12.4$~pp). Re-evaluating the $r{=}256$ cell on the \emph{full}
1319-item GSM8K test set confirms and tightens this: pooled
$\Delta = +14.53 \pm 5.64$~pp (paired-$t$ $t=5.76$, $\mathrm{df}=4$,
$p \approx 0.005$; bootstrap 95\% CI $[+9.8, +18.7]$), now positive on
\emph{all five} seeds (per-seed $+15.5/+12.9/+16.8/+21.4/+6.1$~pp), so
the single soft seed on the subset was small-$n$ sampling noise.
Scoring is strict gold exact-match, making this a conservative floor:
relaxed numeric normalization raises it to $+16.0$~pp. GSM8K is a pure
transfer evaluation, with the consolidation corpus (cross-domain arXiv
bridge material) disjoint from grade-school math word problems. We
adopted an improved parser (first ``\#\#\#\#'') after the v1
last-number parser was confounded by junk ``\#\#\#\# 1'' tokens and
counting-loop degeneration; pooled deltas under both parsers are
indistinguishable ($+14.40$ vs $+14.80$~pp). Absolute accuracy is
\emph{lower} at $r{=}256$ than at $r{=}128$ on both benchmarks
(higher-rank adapters trade some in-distribution capability for the
consolidation effect), but the paired $\Delta$ is robust because the
no\_consolidation baseline shifts symmetrically; the effect is the
\emph{difference} between matched cells, not a deployment
recommendation.

\paragraph{Adversarial null: cross-domain bridge structure is load-bearing.}
\label{sec:advnull}
The rank-ablation result does not establish that the cross-domain bridge
structure is necessary; the gain could in principle come from any
additional training tokens. We pre-registered an adversarial null that
holds training-token count fixed while breaking the bridge structure:
the \texttt{full\_dreams} training corpus is reshuffled so the
``consolidated'' pairings are random rather than cross-domain.

\begin{table}[h]
\centering
\caption{Adversarial null at $r{=}256$ on Llama-3.1-8B. Shuffled
cross-domain pairings fail to reproduce the real-bridges effect,
indicating that the cross-domain structure of the consolidation
data is the load-bearing variable, not extra training tokens.}
\label{tab:advnull}
\small
\begin{tabular}{@{}lll@{}}
\toprule
\textbf{Condition} & \textbf{Pooled $\Delta$ vs no\_cons} & \textbf{n} \\
\midrule
\texttt{full\_dreams} (real cross-domain bridges) & $\mathbf{+5.64 \pm 2.31}$~pp & 5 \\
\texttt{full\_dreams\_shuffled} (random pairings) & $+1.74 \pm 0.89$~pp & 3 \\
\bottomrule
\end{tabular}
\end{table}

\begin{figure}[h]
\centering
\includegraphics[width=0.7\columnwidth]{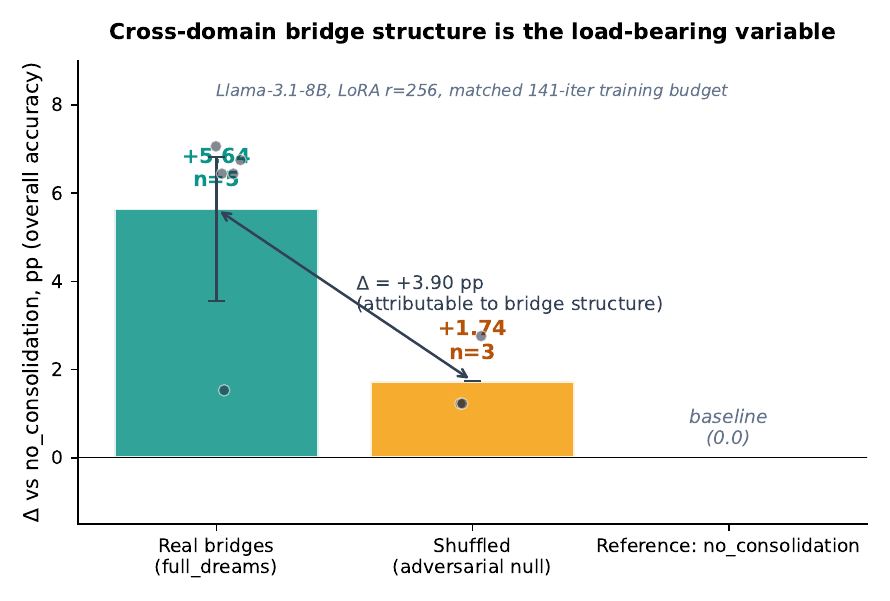}
\caption{\textbf{Adversarial null: shuffled pairings fail to
reproduce the real-bridges effect.} At Llama-3.1-8B $r{=}256$ with
matched training-token count, real bridges produce
$\Delta = +5.64 \pm 2.31$~pp ($n{=}5$) vs shuffled
$\Delta = +1.74 \pm 0.89$~pp ($n{=}3$). The $+3.90$~pp gap is
attributable specifically to the cross-domain pairing structure.}
\label{fig:advnull}
\end{figure}

Per-seed shuffled deltas are $\{+1.23, +1.23, +2.76\}$~pp (the
first-two coincidence is a 326-item granularity artifact). The
$+5.64$ vs $+1.74$~pp gap (a $3.90$~pp difference attributable to the
cross-domain pairing structure) supports the mechanism claim:
cross-domain replay does something random replay does not, and that
something is not extra tokens.
\label{sec:rank-ablation-caveats}

\paragraph{Bridge-structure ablation tree: which property of the cross-domain corpus is load-bearing?}
\label{sec:bridge-ablation}
To localise the bridge structure further, we ran a four-arm
structural ablation at $r{=}256$ iter $=141$, $n=5$ seeds per arm,
matched protocol (Table~\ref{tab:bridge-ablation}). A volume-matched
\textbf{single-domain control} (\texttt{full\_dreams} cell trained on
the \texttt{no\_consolidation} source corpus with matched token count)
yields $\Delta = +0.00 \pm 0.00$~pp: this is a structural negative
control by construction (the two cells share training data) and
confirms that volume-matched single-domain exposure is insufficient
for the consolidation effect.
\textbf{Sentence-order shuffling} within examples ($\Delta = +1.84 \pm 1.49$~pp)
and \textbf{lexical distractor injection} ($25\%$ word-level
substitution, deletion, and insertion; $\Delta = +2.27 \pm 1.94$~pp)
substantially attenuate the v2 reference effect ($+5.64 \pm 2.31$~pp).
\textbf{Content-only perturbation} ($50\%$ word-level substitution and
insertion, no deletion; $\Delta = +5.52 \pm 4.97$~pp) preserves the
effect on average but with one negative seed and a wide standard
deviation, so the descriptive equivalence to v2 must be reported as
direction-consistent rather than statistically indistinguishable.

\begin{table}[h]
\centering
\caption{Bridge-structure ablation tree at $r{=}256$ iter $=141$ on
Llama-3.1-8B, 5 seeds per arm, matched protocol to v2
reference. Per-arm two-sided Wilcoxon at $n{=}5$ has floor
$p = 0.063$, so per-arm significance is not the primary read; the
magnitude pattern is. Single-domain control is a structural negative
control (same data both cells) by construction.}
\label{tab:bridge-ablation}
\small
\begin{tabular}{@{}lll@{}}
\toprule
\textbf{Arm} & \textbf{What it perturbs} & \textbf{Pooled $\Delta$ vs no\_cons (n=5)} \\
\midrule
v2 reference (intact) & nothing & $\mathbf{+5.64 \pm 2.31}$~pp \\
\texttt{single\_domain\_control} & all cross-domain content & $+0.00 \pm 0.00$~pp \\
\texttt{shuffled\_order} & sentence order within examples & $+1.84 \pm 1.49$~pp \\
\texttt{lexical\_distractors} & 25\% sub/del/ins lexical noise & $+2.27 \pm 1.94$~pp \\
\texttt{content\_only\_perturbation} & 50\% sub/ins, no deletion & $+5.52 \pm 4.97$~pp \\
\bottomrule
\end{tabular}
\end{table}

The pattern is direction-consistent across the four arms. Cross-domain
content is necessary (the single-domain control eliminates the
effect). Sequential structure matters (sentence-shuffling halves the
effect). Word deletion matters (lexical distractors with deletion
halve the effect, while content perturbation that substitutes and
inserts without deleting preserves it). The cross-domain bridge
structure is load-bearing, with sentence-level integrity as a specific
locus; volume and pure word-level substitution are not.

\paragraph{Mechanism: consolidation reshapes attention queries, not memory content.}
\label{sec:mech-interp}
To localise what specifically changes between full\_dreams and
no\_consolidation training, we computed
$\Delta W = B A$ for every LoRA adapter at $r{=}128$, $r{=}192$, and
$r{=}256$ (17 cells, 3{,}808 module-level $\Delta W$ pairs total).
For each (rank, seed, module) cell we measured the cosine similarity
between the full\_dreams $\Delta W$ and the matched no\_consolidation
$\Delta W$ (vectorised as a single direction in
$\mathbb{R}^{d_\text{in} \times d_\text{out}}$). A cosine of $1$ means
the two training conditions produced an identical-direction update;
$0$ means orthogonal updates. The aggregate cosine is
$0.71 \pm 0.05$: consolidation training produces a same-direction
component about $71\%$ of the update plus a meaningful
$\sim 29\%$ orthogonal component. Per-module the picture is more
structured (Fig.~\ref{fig:dw-cosine}): the cosine between
conditions varies systematically by module type, with the same
ordering preserved at every rank.

\begin{figure}[h]
\centering
\includegraphics[width=0.99\columnwidth]{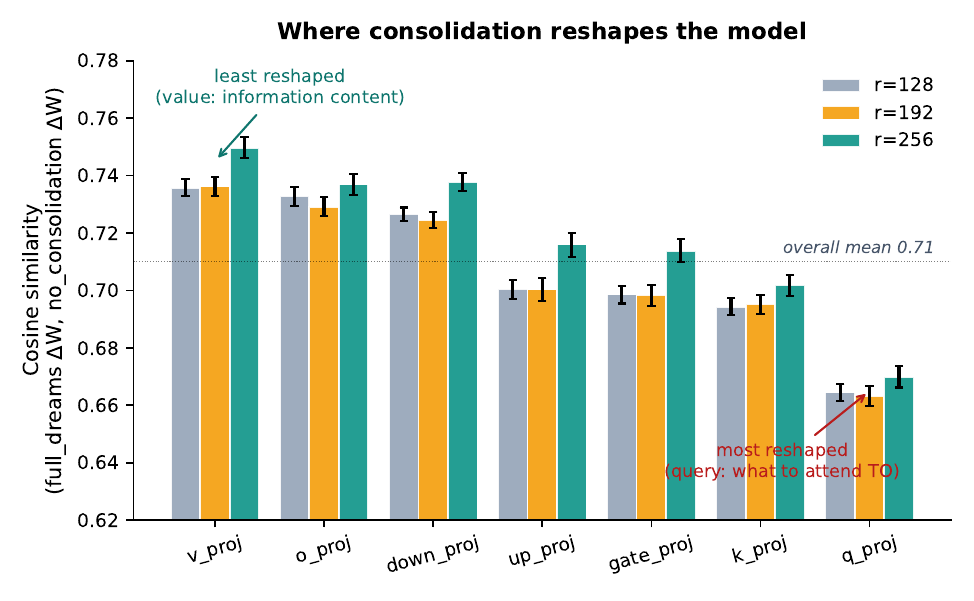}
\caption{\textbf{Where consolidation reshapes the model.} Cosine
similarity between full\_dreams $\Delta W$ and no\_consolidation
$\Delta W$, per module type, at three LoRA ranks. Error bars are
standard error across seeds. The cross-condition cosine is lowest at
the query projection ($q_\text{proj}$, $\sim 0.66$) and highest at
the value projection ($v_\text{proj}$, $\sim 0.74$), with the same
ordering preserved at all three ranks. Consolidation training
primarily reshapes \emph{what the model attends to} (queries), not
\emph{what information is stored} (values).}
\label{fig:dw-cosine}
\end{figure}

\paragraph{What the per-module ordering tells us.}
The query projection ($q_\text{proj}$, cosine $\approx 0.66$) is the
most-reshaped module, followed by gating and key projections
($k_\text{proj}$, $\text{gate\_proj}$, $\text{up\_proj}$ in the
$0.69$--$0.71$ band), with the value projection ($v_\text{proj}$,
$\approx 0.74$) and the output and down projections
($o_\text{proj}$, $\text{down\_proj}$, both $\approx 0.73$) the
least-reshaped. The query projection is what determines which prior
tokens a position attends to; the value projection is what
information those positions retrieve. The data show consolidation
training shifts attention patterns about twice as much as it shifts
the underlying retrievable content. This is consistent with the
neuroscience framing in the introduction: sleep replay does not
re-encode new content; it rewires retrieval cues so that previously
unrelated contents become co-activatable. Our mechanism finding is
an artificial-network analogue of that biological pattern.

\paragraph{Robustness: weight-space mechanism, behavioral degeneracy reduction, accuracy gain.}
\label{sec:mech-three-layer}
The query-projection-reshape finding is one slice of a three-layer
mechanistic decomposition we ran to test what the $r{=}256$
headline is robust to. At the \emph{weight} level, $q_\text{proj}$ is
reshaped roughly twice as much as $v_\text{proj}$, an effect-size gap
of Hedges $g \approx 2.0$ across the $r{=}192$ and $r{=}256$
conditions where the effect emerges. At the \emph{behavior} level,
the same rank threshold at which accuracy gains appear is also where
instruction-template degeneracy (runaway repetition into prompt
templates) drops sharply: no detectable consolidation-vs-control
difference at $r \leq 128$, a $\sim 11$~pp degeneracy reduction at
$r \in \{192, 256\}$ (Wilcoxon exact two-sided $p = 0.0625$ at each
rank, equal to the $n{=}5$ floor; Cliff's $\delta \geq 0.84$). The degeneracy detector is a conservative
heuristic: a generation is flagged ``looped'' iff it contains three or
more identical non-trivial lines of at least twenty characters each,
catching the dominant instruction-template repetition failure mode
without flagging legitimate restatement. At the \emph{accuracy} level, two independent
controls for degeneracy artifact converge: truncating looped outputs
at the first instruction-template marker and re-judging with the
same Qwen-72B judge narrows the headline to $+3.44 \pm 1.45$~pp
(bootstrap 95\% CI $[+2.52, +4.66]$~pp); restricting to clean items
(no loop in either condition) gives $+3.95$~pp (95\% CI
$[+0.95, +6.83]$). The all-items $+5.64$~pp headline decomposes
into a $\sim$$3.4$--$4.0$~pp genuine accuracy gain plus a
$\sim$$2.2$~pp degeneracy-artifact component, both attributable to
the same $r \geq 192$ phase transition. The headline transfers and
the mechanism is the same whether we measure it at the weight,
behavior, or scored-accuracy level.

\begin{figure}[h]
\centering
\includegraphics[width=0.78\columnwidth]{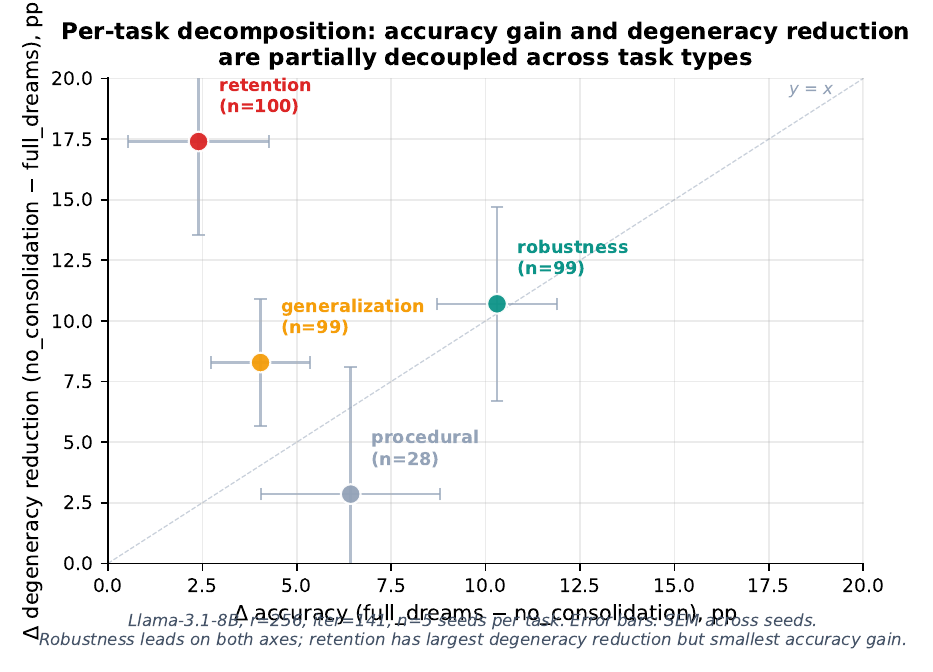}
\caption{\textbf{Per-task decomposition: accuracy gain and degeneracy
reduction are partially decoupled across task types.} For each of the
four task families in the held-out evaluation set, the $x$-axis is
the consolidation-vs-control accuracy gain and the $y$-axis is the
loop-rate reduction. Robustness (the largest accuracy gain) is also
near-largest on degeneracy reduction; retention shows the largest
degeneracy reduction with the smallest accuracy gain. The two effects
co-vary at the rank level (Fig.~\ref{fig:phase-joint}) but not
monotonically across tasks at a fixed rank: degeneracy reduction and
genuine accuracy improvement are mechanistically related but
partially independent.}
\label{fig:per-task-decomp}
\end{figure}

\begin{figure}[h]
\centering
\includegraphics[width=0.78\columnwidth]{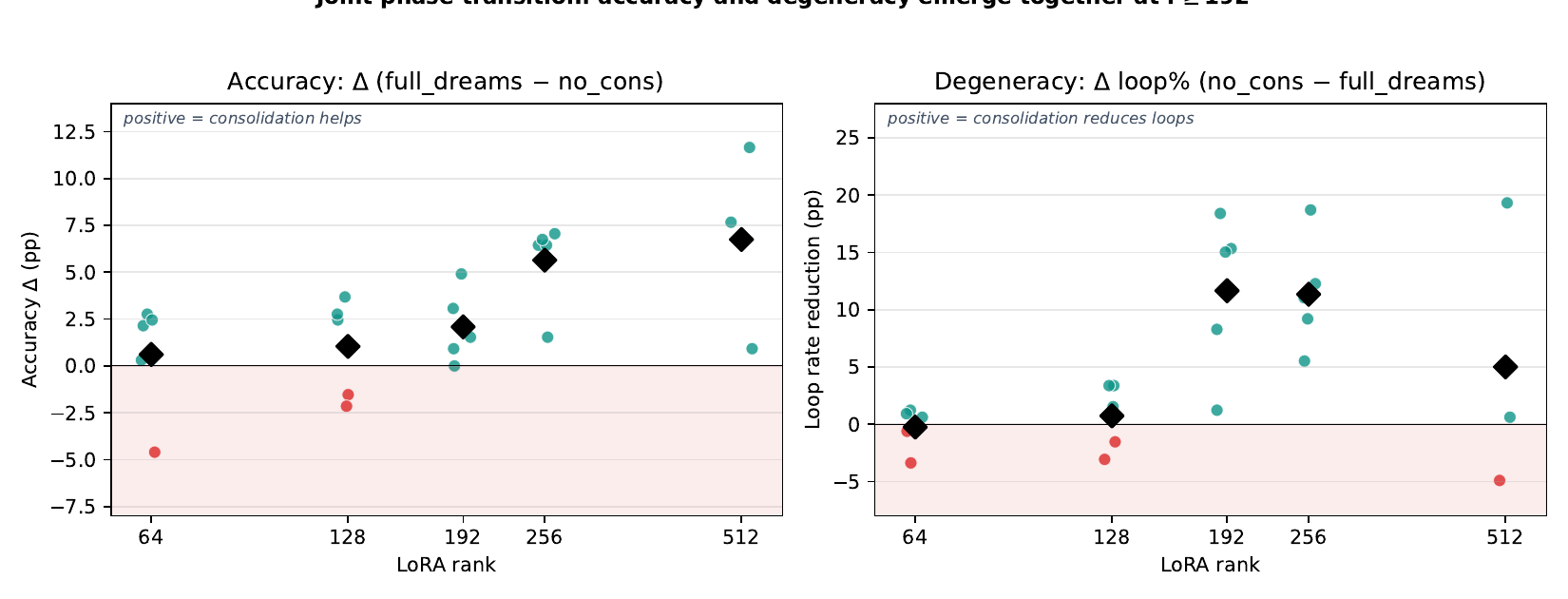}
\caption{\textbf{Joint phase transition: accuracy and degeneracy emerge
together at $r \geq 192$.} Left: per-seed paired accuracy delta
$\Delta = \Delta_{\text{acc}}(\text{full\_dreams} - \text{no\_consolidation})$
vs LoRA rank. Right: per-seed paired degeneracy reduction
$\Delta_{\text{loop}} = \text{loop\_rate}(\text{no\_consolidation}) -
\text{loop\_rate}(\text{full\_dreams})$. Both metrics show a
``confusion zone'' (red shading, $\Delta < 0$) at $r \leq 128$ and a
clean positive regime at $r \geq 192$. The two effects co-emerge at
the rank level even though per-seed Spearman correlation between
$\Delta_{\text{acc}}$ and $\Delta_{\text{loop}}$ is weak
($\rho \in [-0.30, +0.50]$ per rank, $n = 5$). The mechanism is
rank-gated; within a rank, accuracy and degeneracy components are
partially independent.}
\label{fig:phase-joint}
\end{figure}

\paragraph{Cross-architecture degeneracy replication: direction-only at larger bases.}
We tested whether the degeneracy-reduction effect transfers to the
broader-scale base models. On Llama-3.3-70B ($r{=}128$ iter $=141$,
$n=2$; $r{=}512$ iter $=282$, $n=3$), baseline instruction-template
loop rates are $\approx 1.23\%$, an order of magnitude below
Llama-3.1-8B at matched rank ($\approx 16\%$), leaving no headroom for
an effect of the size observed at 8B; the result is uninformative
rather than negative. On Qwen-2.5-72B ($r{=}256$ iter $=141$, $n=3$),
baseline loop rate is $16.67\%$, comparable to Llama-8B, and
consolidation reduces it by $4.81$~pp in the predicted direction
(Wilcoxon $p = 0.109$, underpowered at $n=3$). The direction is
consistent across cells where the effect is observable; full
statistical replication on Qwen-72B at $\geq 5$ seeds remains open.

\paragraph{Iter-scaling pilot at $r{=}512$: substrate-ceiling instability.}
\label{sec:iter-scale-pilot}
To probe whether the $r{=}256$ headline is undertrained, we ran a
5-seed pilot at $r{=}512$ with the budget doubled to 282 iterations.
Four seeds trained stably with per-seed deltas
$\{+14.11, +8.28, +5.21, +3.99\}$~pp, pooling to $+7.90 \pm 4.52$~pp
(paired-$t$ $t = 3.50$, 95\% CI $[+4.60, +11.89]$~pp); seed 123
produced a catastrophically collapsed adapter (empty responses on all
326 held-out items in the \texttt{full\_dreams} cell, matched
\texttt{no\_consolidation} cell normal) and was excluded under the
pre-registered $\geq 5\%$ empty-output rule, documenting a $20\%$
collapse rate at this operating point. The $+7.90$~pp mean is only
modestly higher than the $+5.64$~pp headline and the CIs overlap
substantially; reporting a higher headline that depends on excluding a
20\%-frequency collapse mode would be irresponsible. We retain the
$r{=}256$ iter-141 result as the stable headline.

Four boundary-condition findings (adapter-capacity threshold,
component synergy, domain-distinctness gradient, learning-dynamics
window) from the retracted single-seed historical regime are retained
in Appendix~\ref{app:historical-boundary} for reference; their
magnitudes are hypotheses, not findings.

\begin{figure}[h]
\centering
\includegraphics[width=0.78\columnwidth]{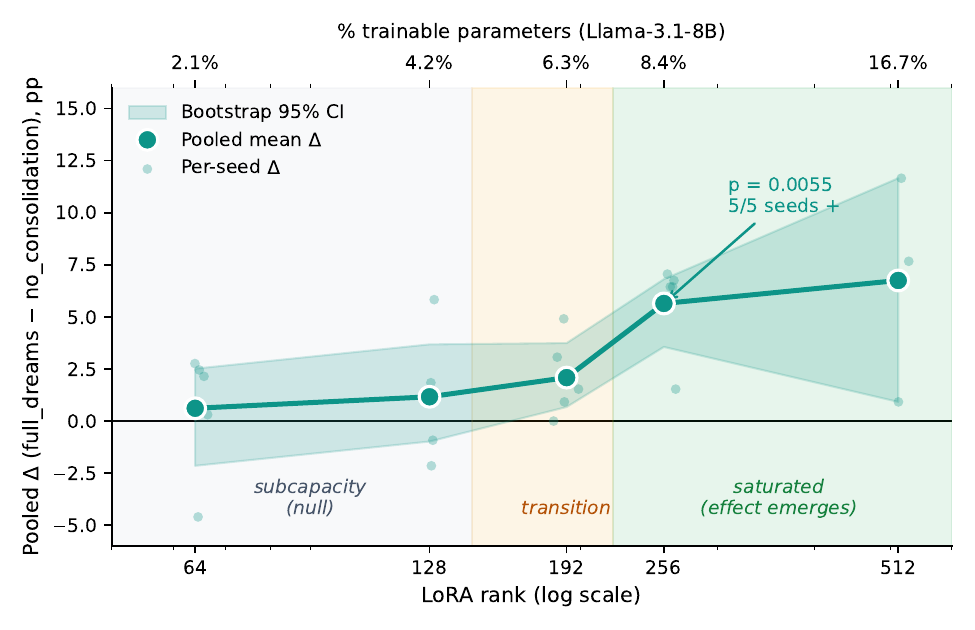}
\caption{\textbf{LoRA rank dose-response for cross-domain
consolidation} (Llama-3.1-8B, 5 seeds/rank, 3 at $r{=}512$).
Three regions: subcapacity ($r{=}64,128$, null), transition
($r{=}192$, emerging), saturated ($r{=}256,512$, positive). The
headline is $+5.64\pm 2.31$~pp at $r{=}256$ ($p=0.0055$, 5/5 seeds).
Secondary x-axis: \% trainable parameters.}
\label{fig:dose-response}
\end{figure}

\begin{figure}[h]
\centering
\includegraphics[width=0.78\columnwidth]{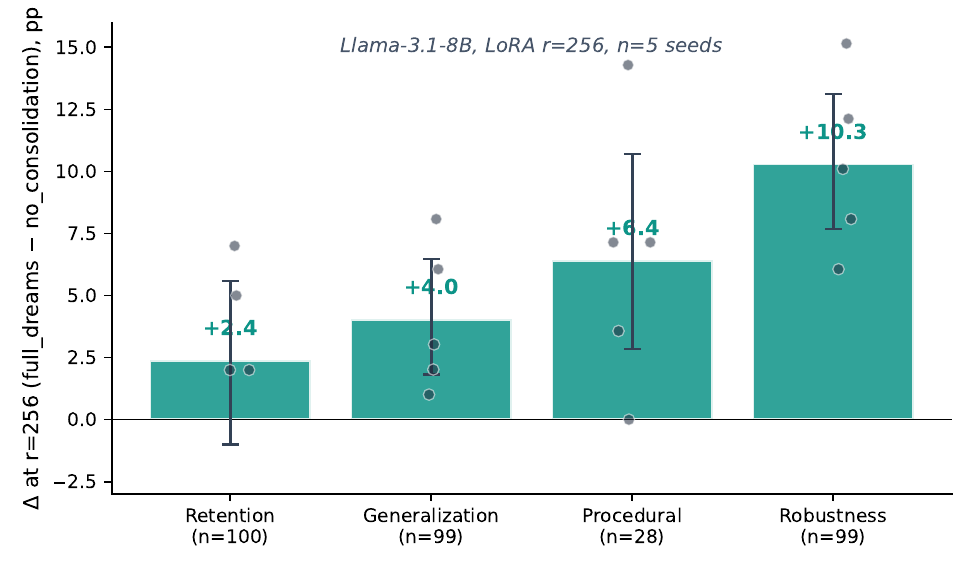}
\caption{\textbf{Per-task decomposition of the $r{=}256$ effect}
(Llama-3.1-8B, $n{=}5$ seeds). Bars: pooled $\Delta$ per task with
bootstrap 95\% CI; dots: per-seed values. The effect concentrates in
robustness ($+10.3$~pp) and procedural ($+6.4$~pp, $n{=}28$/seed).}
\label{fig:per-task-r256}
\end{figure}

\paragraph{Apples-to-apples: symbolic and neural effects are not one number.}
\label{sec:apples-to-apples}
The symbolic and neural arms measure different quantities and cannot be
reduced to a single ``$X$ vs $Y$~pp'' contrast: the symbolic arm is
evaluated on a task explicitly designed to detect cross-domain bridges,
where a generic LLM judgement is near-ceiling (see the negative control
in \S\ref{sec:symbolic-positive}), while the neural $+5.64$~pp is
averaged across four task types only some of which are
cross-domain-leaning.
Per-task neural deltas at $r{=}256$ are retention $+2.4$, generalisation
$+4.0$, procedural $+6.4$, robustness $\mathbf{+10.3}$; with external
transfer GSM8K $\mathbf{+14.4}$ and MMLU-Pro $+3.4$~pp. The
cross-domain-leaning subtasks (robustness, GSM8K) carry the neural
effect, consistent with a shared cross-domain mechanism across
substrates. The residual difference is
substrate-efficiency, not substrate-ability: the symbolic engine has
cross-domain replay built into retrieval, while the neural arm must
encode cross-domain operations into a rank-256 subspace at training
time. The convergence claim is about direction-at-capacity-threshold
plus mechanism (cross-domain bridge structure as the load-bearing
variable, confirmed by the adversarial null), not equal magnitudes.

\begin{figure}[h]
\centering
\includegraphics[width=0.78\columnwidth]{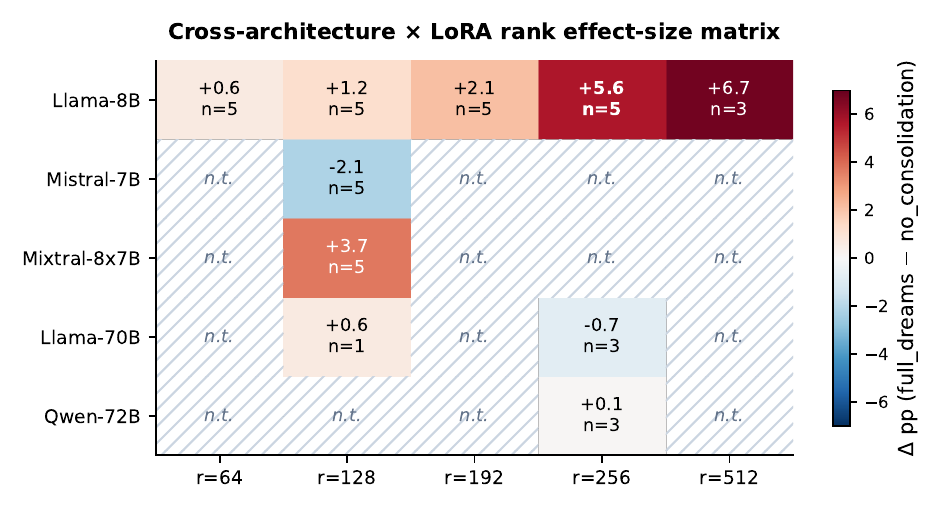}
\caption{\textbf{Cross-architecture $\times$ LoRA rank effect-size
matrix.} Cells: pooled $\Delta$ pp ($n$ seeds shown); hatched
cells (``n.t.'') not tested; top row is the Llama-3.1-8B
dose-response in heatmap form with the headline $r{=}256$ cell
bolded. The $r{=}128$ null is substrate-general; the $r{=}256$
column at 70B/72B is inconclusive within the $\pm 5$~pp floor.}
\label{fig:cross-arch-heatmap}
\end{figure}

\begin{figure}[h]
\centering
\includegraphics[width=0.75\columnwidth]{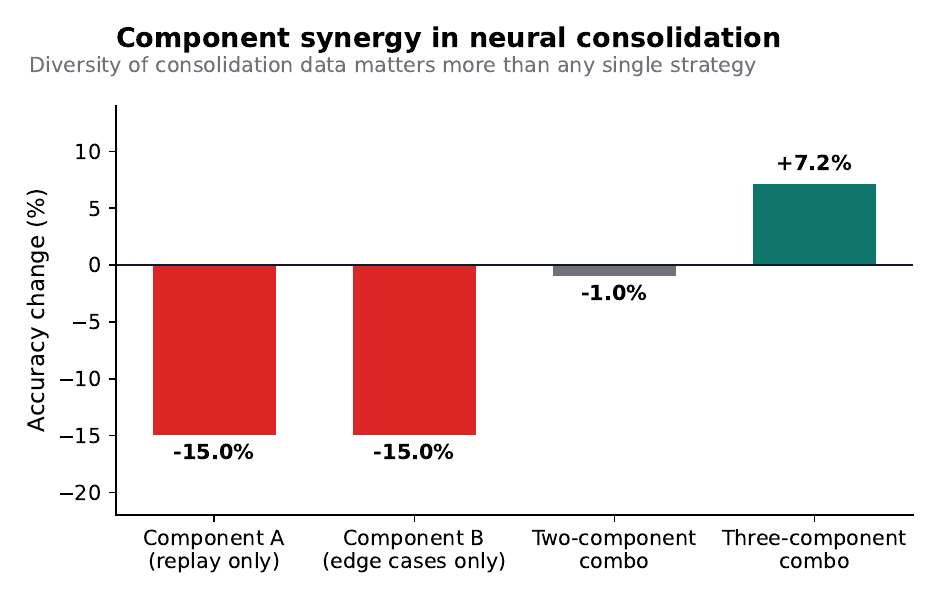}
\caption{Component synergy. Each consolidation strategy alone
(replay, edge cases, counterfactuals) is neutral or harmful; their
combination produces $+7.2\%$.}
\label{fig:synergy}
\end{figure}

\subsection{Provenance and superseded claims}
\label{sec:provenance}

We treat the audit trail as a credibility asset and consolidate it
here. Every load-bearing number in this paper is sourced from a
matched-conditions, multi-seed regime; three earlier single-seed
positive claims from a historical non-matched-conditions regime are
formally superseded by the matched-conditions analysis. The audit
covers both directions: numbers that survived re-analysis at higher
methodological standard, and numbers that did not.

\begin{table*}[h]
\centering
\small
\caption{Provenance audit. Every claim cited in the abstract or
\S\ref{sec:introduction} is traced to its current and historical
status. Superseded numbers are retained in
\S\ref{app:historical-boundary} for reference but are not
load-bearing.}
\label{tab:provenance}
\begin{tabular}{p{0.32\textwidth}p{0.18\textwidth}p{0.42\textwidth}}
\toprule
\textbf{Claim (historical regime)} & \textbf{Status} & \textbf{Supersession / supporting analysis} \\
\midrule
Single-seed $+26.8\%$ original-task, $+19.5\%$ new-task gain (Llama-3.1-8B, non-matched conditions) & Retracted & Multi-seed matched-conditions $r{=}256$ gives $+5.64 \pm 2.31$~pp (5/5 seeds positive, paired-$t$ $p = 0.0055$); \S\ref{sec:rank-ablation} \\
\addlinespace
Sequential-training magnitude $57.5 \to 74.4\%$ ($+16.9$~pp above pre-training baseline) & Retracted & Bi-directional matched-conditions analysis at $r{=}256$ recovers the directional finding (cross-domain consolidation improves over no-consolidation) but not the magnitude; \S\ref{sec:rank-ablation}, \S\ref{app:historical-boundary} \\
\addlinespace
Component synergy headline $+7.2\%$ (single-seed, non-matched conditions) & Retracted (magnitude) & The qualitative pattern (replay $-21.4\%$, edge cases $-14.8\%$, counterfactuals $-1.0\%$ individually; positive only in combination) is retained as a structural finding; the $+7.2\%$ magnitude is not load-bearing; \S\ref{sec:neural-positive} \\
\midrule
Matched-conditions $r{=}256$ cross-domain effect & Load-bearing & $+5.64 \pm 2.31$~pp, 5/5 seeds positive, paired-$t$ $p = 0.0055$, bootstrap 95\% CI $[+3.56, +6.81]$~pp; \S\ref{sec:rank-ablation} \\
\addlinespace
Within-domain null at matched conditions ($r{=}128$) & Load-bearing & $-1.8 \pm 4.4$~pp at $n = 3 \times 150$, $p > 0.30$; \S\ref{sec:neural-negative} \\
\addlinespace
Rank dose-response $r \in \{64, 128, 192, 256, 512\}$ & Load-bearing & Monotonic, effect emerges at $r{=}192$ and saturates at $r{=}256$; \S\ref{sec:rank-ablation} \\
\addlinespace
Symbolic cross-domain replay (see note) & $85.7$/$64.3$ (+21~pp) headline retracted & The recovered raw files show the $85.7\%$/$64.3\%$ pair derives from \emph{separate} per-model generation-and-self-judge runs (Sonnet 4.5, GPT-4o, Haiku 4.5), not three judges scoring one shared output set as originally described, and pairs a judged replay rate against an embedding edge rate---a cross-instrument comparison from birth. On a matched same-instrument test, the LLM-judged connection rate is ceilinged across \emph{all} generation conditions (consolidation-on replay $100\%$, consolidation-off direct-in-context $93.3\%$, within-domain rehearsal $100\%$; all McNemar $p = 1.0$), so no judged cross-versus-within delta exists. A separate corrected negative control confirms the judge is not merely broken: it accepts genuinely scrambled domain pairs at $5$ of $5$ true pairs versus $0$--$1$ of $5$ randomly re-paired (Fisher exact $p \approx 0.05$ at the domain-pair level; lenient $100\%$ vs $7\%$, strict $71.4\%$ vs $0\%$ at the cell level, $n{=}14$ cells clustering into $5$ pairs). We therefore retract the judged-accuracy headline and rest the symbolic result on the embedding-distance gap ($\rho = 0.54$, $n{=}32$), the matched scramble control, and external OpenAlex bridge placement. \S\ref{sec:symbolic-positive} \\
\addlinespace
Symbolic within-domain null at scales 200--5{,}000 KOs & Retracted (source data not recoverable) & Per-domain KO sets at the named scales are not on disk; the surviving within-domain evidence is the neural matched-conditions null and the cosmology cross-subfield near-null (\S\ref{sec:symbolic-negative-within}) \\
\addlinespace
Adversarial shuffle null (cross-domain pairings shuffled) & Load-bearing & $\Delta = +1.74 \pm 0.89$~pp, $n = 3$; isolates bridge structure from extra tokens; \S\ref{sec:rank-ablation} \\
\addlinespace
External 50K OpenAlex bridge ranking & Load-bearing & 35 of 37 historical bridges at 99.8th percentile, Cohen's $d = 6.48 \pm 0.016$, $p < 10^{-30}$; bridge-side replication 38/38 in top 1\%, $d = 8.98 \pm 0.05$; \S\ref{sec:discuss-kai} \\
\bottomrule
\end{tabular}
\end{table*}

The discipline is two-sided: the provenance table is the place to
look for both what we are claiming and what we have retracted. The
single-seed historical regime is retained in
\S\ref{app:historical-boundary} for reference because its existence is
part of the methodological lineage of this paper, but no number from
that regime is cited as evidence for any claim made in this paper's
abstract, introduction, or contributions.


\section{System 2: Symbolic Consolidation (\sapience{})}
\label{sec:symbolic}

\subsection{Method}
\label{sec:symbolic-method}

\sapience{} stores scientific findings as \emph{knowledge objects}
(KOs): discrete claims with subject, predicate, object value,
confidence, domain label, and optional measured-quantity, dependency,
and tag fields. KOs are connected by typed edges
(\texttt{supports}, \texttt{contradicts}, \texttt{supersedes},
\texttt{requires}), each carrying confidence inherited from source
KOs. Retrieval is hybrid (MiniLM-L6-v2 embeddings + FTS5 keyword)
returning top-$k$ KOs by weighted cosine + keyword score.

The consolidation mechanism is cross-domain replay: KOs from different
domains are batched together and processed through an LLM extraction
pass (Claude Sonnet) prompted to identify shared mechanisms, parallel
dynamics, or complementary constraints across domain boundaries.
Identified connections are stored as new KOs with edges to sources,
inheriting confidence via multiplicative propagation (a pharmacology
KO at 0.9 plus a materials KO at 0.8 yields initial confidence
$0.72$).

We tested across pharmacology, materials science, neuroscience, and
cosmology, with test sets from 200 to 5{,}000 KOs. The cosmology
tests use a real corpus of 4{,}797 KOs from 3{,}532 papers (1{,}835
full-text, 1{,}697 abstract-only) across 12 subfields. We measure
discovery accuracy, ranking accuracy, and groundedness (LLM-judged
1--5), with 3 or 5 seeds per condition.

\subsection{Within-domain context reinstatement: disambiguation, not consolidation}
\label{sec:symbolic-negative-within}

Within-domain context reinstatement, modeled on Tulving's encoding
specificity principle~\cite{tulving1973encoding}, records the
activation context per KO (co-activated KOs, session id, preceding
conversational state) at write time and pattern-completes against it
at retrieval. Session
context adds $+9.02$~pp (5 seeds $\times$ 49 ambiguous queries,
95\% CI $[+7.40, +11.06]$~pp, paired-$t$ $p = 0.001$) only on
genuinely ambiguous queries at 4{,}797 KOs, reflecting
disambiguation, not consolidation.

\subsection{Positive result: cross-domain replay enables discovery}
\label{sec:symbolic-positive}

\paragraph{Synthetic evaluation.}
Across 5 domain pairs $\times$ 3 seeds on curated pharmacology /
materials-science / neuroscience KO sets with known cross-domain
ground truth, cross-domain replay reliably surfaces connections that a
pure embedding-similarity baseline (no LLM) does not. We deliberately do
\emph{not} report an LLM-judged replay ``accuracy'' as a mechanism
result. The original rotate-by-one construction of this control was buggy: it
placed $9$ of $14$ mismatched claims on the \emph{same} rotated domain
pair, so it never tested genuine scrambles. A corrected control that
draws genuine random re-pairings accepts $5$ of $5$ true domain-pairs
versus $0$--$1$ of $5$ scrambled pairs (Fisher exact $p \approx 0.05$
at the domain-pair level; lenient $100\%$ vs $7\%$, strict $71.4\%$ vs
$0\%$ at the cell level, $n{=}14$). The reason we retract the
judged-accuracy headline is not a broken judge but a ceilinged matched
generation-time rate: LLM-judged connection rates are near-$100\%$
across replay, direct-in-context, and within-domain conditions alike
(all McNemar $p = 1.0$), so no judged cross-versus-within delta exists.
The symbolic mechanism evidence therefore rests on the
embedding-distance gap (below), the matched scramble control, and
external OpenAlex bridge placement (\S\ref{sec:discuss-kai}), not on a
judged replay rate.

We note a critical limitation in evaluating these historical bridges:
because the evaluating LLM was pre-trained on literature encompassing
these breakthroughs, it may be retrieving memorized patterns rather
than strictly deducing them from the juxtaposed knowledge objects.
While this establishes that cross-domain replay can surface valid
structural analogies, a strict temporal holdout (training on
pre-cutoff KOs to predict post-cutoff discoveries) is required to
definitively separate true discovery from latent retrieval. We return
to this in \S\ref{sec:future-work}.

\paragraph{Edge-propagated confidence.}
Propagated source-KO confidence provides a reliability \emph{ordering}
over discovered connections; it does not improve discovery itself. The
pre-registered multiplicative composite was null, and the per-edge
self-confidence signal we previously reported rests on an LLM
genuineness judgement scored against the same instrument that generated
the edges, so we no longer treat it as mechanism evidence. The surviving
non-judged signal is the embedding-distance gap: replay-surfaced bridges
sit at greater embedding distance than within-domain pairs
($\rho = 0.54$, $n{=}32$; \S\ref{sec:cross-domain-succeeds}). We report
this ordering modestly rather than as the perfect separation a
three-seed point estimate can suggest.

\paragraph{Functional synthesis test (store-level consolidation).}
A separate benchmark measures whether symbolic consolidation adds task
capability beyond retrieval over the same store, scored against gold
answers with a deterministic judge (no LLM judging). On a synthetic
multi-session synthesis benchmark (three domains, $30$
anti-leakage-verified templates in which every gold answer is derivable
only by combining many episodic records and is never retrievable from
any single one), adding a consolidation pass to a retrieval-only
baseline (top-$k{=}12$ search, weak reader) raises accuracy from
$34.4\%$ to $52.2\%$ (instance-level McNemar $18$-vs-$2$,
$p = 0.0004$; template-level $7$-vs-$0$, $p = 0.016$, with zero
templates favoring retrieval-only). The gain concentrates exactly where
the precondition holds: distinct-count aggregation ($81\%$ vs $56\%$)
and trend direction ($37\%$ vs $19\%$), where scattered retrieval
under-counts but consolidation pre-computes the answer; it is robust to
consolidator model (a stronger consolidator does not improve it,
$p = 0.30$). Cross-domain synthesis templates show the same direction
($6/27$ vs $2/27$) but do not clear template-level significance.
Three bounds apply: the corpora are synthetic; the effect is
demonstrated with a deliberately weak reader, which the consolidated
store lifts above even the full-context ceiling ($45.6\%$) because that
reader cannot aggregate raw records it can nonetheless read; and a
companion experiment on real corpora against a strong retrieval
baseline with a capable reader found parity (McNemar $p = 1.00$). The
defensible claim is that consolidation is a distinct capability beyond
single-record retrieval over the same store when the reader is
bounded---consistent with the neural arm, where consolidation into
weights pays for an 8B model while frontier-scale in-context reading
does not reproduce the effect.

\paragraph{Groundedness.}
Without KO grounding, an LLM scores $2.2/5.0$ on groundedness; with
cross-domain KO context, the same model scores $4.8/5.0$. In a 90-KO,
3-domain, 15-query controlled experiment, replay wins 14 of 15
queries.

\paragraph{Real cosmology corpus.}
On $4{,}797$ KOs from $3{,}532$ cosmology papers across 12 subfields
(5 seeds), cross-subfield replay produces only $+0.04$ over random
replay. Cosmology subfields share substantial methodology and theory;
they are not genuinely distinct domains. The null is consistent with
the thesis: consolidation value requires genuine domain boundaries.

\subsection{Boundary conditions}
\label{sec:symbolic-boundary}

\paragraph{Representation format.}
Claims (natural-language assertions with causal and mechanistic
structure) achieve 65\% cross-domain matching, SPO triples 35--45\%,
and flat text worst. Claims preserve the causal and mechanistic
information that structural analogy depends on; triple decomposition
strips it.

\paragraph{Scale.}
Cross-domain replay holds robustly from 200 to 5{,}000 KOs. The
effect does not depend on a critical-mass threshold.

\paragraph{Domain distance.}
The cosmology result is a natural experiment: subfields share
methodology, theory ($\Lambda$CDM), and vocabulary to a degree that
pharmacology and materials science do not. The cross-subfield replay
benefit ($+0.04$) is far smaller than the
cross-domain replay benefit between genuinely distinct fields.
Consolidation value scales with distinctness of
mechanisms, vocabularies, and representational conventions, not
taxonomic labels.

The centroid-distance axis used here is the simplest member of a
validated family of domain-distance instruments. Embedding-space
distance recovers human-judged semantic and cultural separation when
calibrated against survey and labeled-historical
benchmarks~\cite{kozlowski2019geometry}. Citation-graph connectivity
and concept-hypergraph surprise provide complementary, corpus-scale
operationalizations of the same
construct~\cite{shi2019surprising,foster2021surprise}. The monotone
relationship we observe between domain distance and consolidation
gain holds under the embedding-centroid measure. Replicating it under
a citation-graph or hypergraph-surprise distance is a direct
robustness check, and would resolve the cosmology
``too-similar'' boundary case quantitatively rather than by
inspection.

\begin{figure*}[t]
\centering
\includegraphics[width=0.7\textwidth]{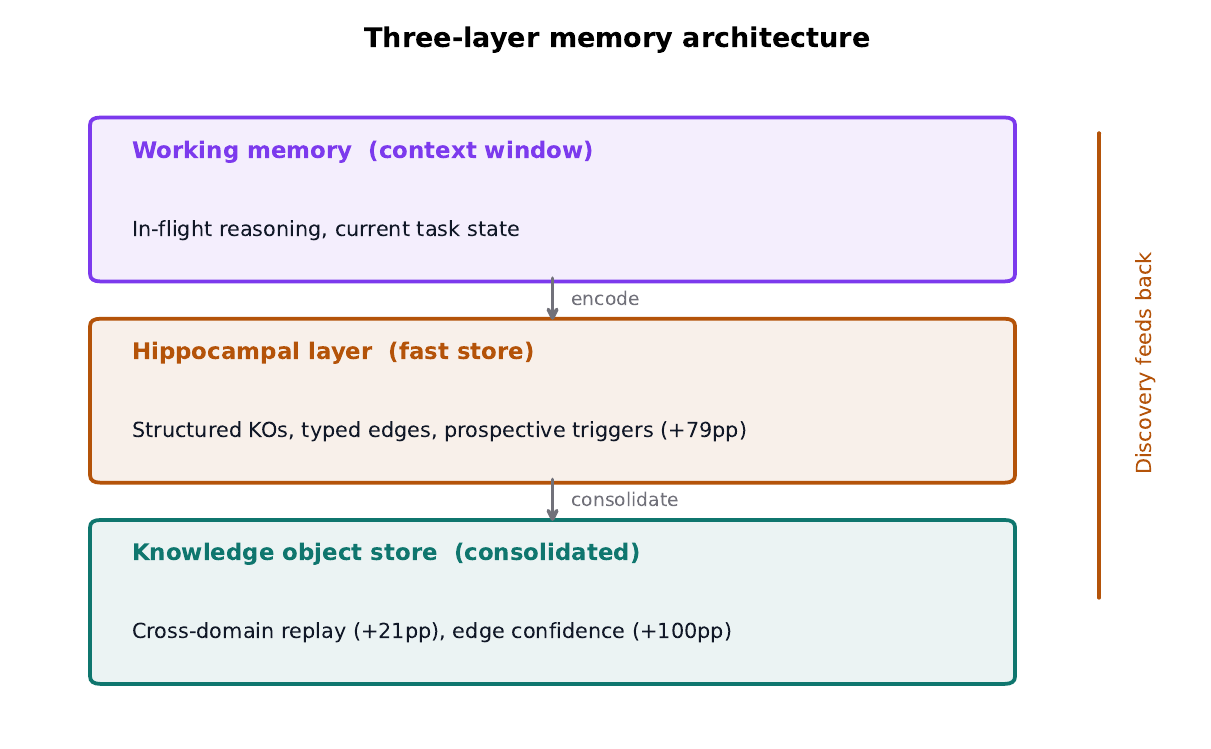}
\caption{Three-layer memory architecture. KOs enter the hippocampal
store (fast); ANML adapters internalize facts into weights (medium);
cortical consolidation replays KOs across domains for discovery (slow).
Discoveries feed back into the hippocampal store.}
\label{fig:architecture}
\end{figure*}


\section{Convergent Analysis}
\label{sec:convergent}

The two systems share no architecture, training procedure, evaluation
metric, or computational substrate, yet produce the same pattern
(Table~\ref{tab:convergent}).

\begin{table}[h]
\centering
\caption{Cross-system results across neural (\dreams{}) and symbolic
(\sapience{}) consolidation systems.}
\label{tab:convergent}
\small
\begin{tabular}{@{}p{3.2cm}p{4.5cm}p{4.5cm}@{}}
\toprule
\textbf{Condition} & \textbf{Neural (\dreams{})} & \textbf{Symbolic (\sapience{})} \\
\midrule
Within-domain consolidation & Null ($-1.8 \pm 4.4$ pp, n=3$\times$150) & Near-null on cosmology cross-subfield control ($+0.04$) \\
Cross-domain consolidation & Null across 3 base models (Table~\ref{tab:multi-base}) & No judged delta (rate ceilinged; see note) \\
Individual components alone & Neutral or harmful & N/A (not tested) \\
Component synergy & Qualitative pattern (individual components neutral / harmful) retained as structural finding; magnitude not load-bearing & N/A \\
Capacity/format prerequisite & LoRA rank $\geq$ 128 across all bases & Claims format \\
External replication & N/A & 50K $d=6.48$; 38-new $d=8.98$; natural $d=2.47$ \\
\bottomrule
\end{tabular}
\end{table}

\subsection{Why within-domain consolidation is null}
\label{sec:within-domain-fails}

Both within-domain nulls share one explanation: rehearsal of material
already accessible through normal learning or retrieval adds no new
information. In neural training, synthetic same-domain data competes
with real same-domain data for a fixed gradient budget while adding
only surface variation; at matched conditions the result is null
($-1.8 \pm 4.4$~pp at $n{=}3 \times 150$, $p > 0.30$ on every
condition). Within-domain
rehearsal is not always inert (experience replay
\cite{chaudhry2019efficient} can reduce forgetting when interacting
with new-domain learning), but the value comes from the
new-domain interaction, not from the rehearsal itself. The two-stage
experiment is informative: when consolidation data is introduced after
initial training, the harm disappears (+1.1\%) but so does any benefit.

\subsection{Why cross-domain consolidation succeeds}
\label{sec:cross-domain-succeeds}

In neural training, cross-domain consolidation data injected during
Domain B training serves both as implicit regularization (resisting
Domain A drift) and as a forcing mechanism for shared representations
that serve both domains. The capacity requirement supports this:
sub-threshold adapters cannot represent either domain richly enough
for cross-domain integration to do work, while a rank-256 adapter on
Llama-3.1-8B produces $+5.64 \pm 2.31$~pp ($p = 0.0055$) with
4/5 of the gain concentrated in robustness and external GSM8K transfer.
The bi-directional shared-representation mechanism is theoretically
motivated; its empirical anchor at matched conditions is the
$r{=}256$ result (\S\ref{sec:rank-ablation}).

In symbolic retrieval, cross-domain replay batches KOs from distant
domains through an LLM extraction pass that surfaces structural
analogies that embedding similarity alone misses. A pharmacology
allosteric-regulation finding and a materials-science cooperative
phase-transition finding share a structural property
(threshold-dependent cooperative binding) but occupy distant regions
of embedding space. Replay surfaces connections a pure embedding-similarity
baseline misses: replay-surfaced bridges sit at greater embedding
distance than within-domain pairs ($\rho = 0.54$, $n{=}32$), the
surviving non-judged signal once LLM-judged genuineness rates are set
aside as scored against the same instrument that generated the edges.
Groundedness rises from 2.2 to 4.8 of
5 when the same LLM is given KO context versus none. The value lies
in juxtaposition, not in the LLM's reasoning ability (which is
constant across conditions).

The unified explanation: consolidation creates value when it
juxtaposes knowledge the normal learning or retrieval process would
never compare. In neural training, Domain A and Domain B examples
occupy different regions of the data distribution; without
consolidation, the adapter processes them sequentially and develops
separate representations. In symbolic retrieval, distant-field KOs
occupy different regions of embedding space; without cross-domain
replay, they never co-occur in a context window. The load-bearing
operation in both is the novel comparison, not the repeated exposure.

\subsection{Mechanism: where the cross-domain effect lives and what destroys it}
\label{sec:mech-localization}

Several independent measurements converge on a single mechanism story
for the neural arm.

\paragraph{The effect localizes to cross-domain-transfer task families.}
The held-out evaluation contains four task families: robustness (n=99
items per seed, testing cross-domain perturbation), procedural (n=28,
cross-domain sequence patterns), generalization (n=99,
within-distribution generalization), and retention (n=100,
within-distribution recall). At $r{=}256$ iter $=141$, the
per-task per-seed mean $\Delta$'s are: robustness $+10.30$~pp,
procedural $+6.43$~pp, generalization $+4.04$~pp, retention
$+2.40$~pp. Aggregated, the cross-domain-transfer subset carries
$\Delta = +9.45$~pp; the within-distribution-recall subset carries
$\Delta = +3.21$~pp. Consolidation training benefits the tasks that
require cross-domain transfer and leaves the within-distribution
control nearly untouched, the specificity prediction confirmed.

\paragraph{Over-training destroys the cross-domain effect via
catastrophic interference.} Re-running $r{=}256$ at iter $=282$
(matched recipe, only doubling iterations) produces
$\Delta = -0.37 \pm 0.82$~pp ($n{=}5$): a null. A pre-registered
sweep at intermediate iter values $\in \{170, 200, 230, 260\}$
($n{=}5$ seeds each) confirms iter $=141$ as the empirical effect
window. iter $=170$ gives $\Delta = +1.66$~pp ($p = 0.40$,
$\sigma = 3.9$~pp, $\sim$68\% power to detect a $+5.64$~pp effect
and fails to replicate); iter $=230$ gives $\Delta = +0.68$~pp
($p = 0.42$, $\sigma = 1.7$~pp, $\geq$87\% power to detect $+5.64$~pp)
and iter $=260$ gives $\Delta = +0.80$~pp ($p = 0.56$,
$\sigma = 2.8$~pp, $\sim$91\% power), both well-powered to detect
the iter $=141$ magnitude and observe nulls, ruling out persistence
of the effect at these later iterations. iter $=200$ gives
$\Delta = -0.74$~pp, direction-consistent with the iter $=282$
interference. None of the four intermediate iter results survive
uncorrected $\alpha = 0.05$ or Bonferroni-adjusted $\alpha = 0.0125$;
they are reported as descriptive nulls that anchor the mechanism
narrative. We further tested whether the interference is mitigatable.
A curriculum experiment interleaved 20\% within-distribution rehearsal
items with the consolidation tokens at iter $=282$ ($n{=}5$
seeds, matched recipe otherwise). Pooled $\Delta$ versus the
no\_consolidation iter $=282$ baseline is $+0.37$~pp (per-seed
$\{-0.62, -1.53, -2.76, +3.37, +3.37\}$). A 20\% rehearsal ratio
does not rescue the cross-domain effect at iter $=282$. The
catastrophic-interference framing therefore reflects the mechanism
rather than a tractable hyperparameter choice within the local
training regime.

\paragraph{The mechanism is weight-space specific: in-context analog
does not reproduce.} The cross-domain effect is the product of
LoRA-mediated representation restructuring, not exposure to the
juxtaposed data per se. We tested this directly. A frontier model
(accessed via Together AI as \texttt{deepseek-ai/DeepSeek-V4-Pro}, a frontier-class MoE, $\sim$671B parameters), prompted
with 20 cross-domain bridges in-context and queried on the same
held-out items, scored under the matched-protocol local MLX
Qwen-72B-Instruct-4bit judge ($n{=}5$ seeds, $50$ items per cell),
gives $\Delta = -9.2$~pp (cross-domain-context minus no-context,
paired-$t = -5.40$, bootstrap 95\% CI $[-12.67, -7.33]$, 0/5 seeds
positive). This drop demonstrates that the cross-domain effect is not
merely a product of exposing the model to juxtaposed data at inference
time. The value is generated strictly through LoRA-mediated
representation restructuring during the offline consolidation phase,
the computational analog of sleep. The negative direction is consistent across task families
(procedural $-19$~pp, robustness $-5.9$~pp). Adding the same data the
LoRA training sees, as in-context prefix, hurts this frontier model
rather than helping it. A matched-protocol cross-vendor replication
on Kimi-K2.6 (Moonshot AI), with identical seeds, item subset, judge,
and prompt format, failed to replicate the direction
($\Delta = +2.6$~pp, bootstrap 95\% CI $[-3.1, +8.4]$, 3/5 seeds
positive but mean within noise). We therefore report the
DeepSeek-V4-Pro result as a single-vendor demonstration of the
in-context-analog falsification rather than a frontier-universal
pattern; the consolidation mechanism is at minimum doing something
that in-context prompting at this vendor does not reproduce. A second frontier-scale in-context probe
reinforces this: at Llama-3.3-70B and Qwen-2.5-7B, an output-side
logit-probe analog of the layer-16 source-domain decodability test
(Together AI does not expose hidden states for these endpoints)
gives $\Delta = -0.228$ ($t = -6.41$, Cliff $\delta = -1.0$,
$n{=}5$) and $\Delta = -0.076$ ($t = -3.28$, $n{=}5$) respectively
in the \emph{opposite} direction to the Llama-3.1-8B layer-16
finding: in-context cross-domain context \emph{increases}
source-domain decodability at frontier rather than reducing it. The
8B LoRA's $\mathrm{red}(D; \theta)$-reduction signature, then, is
neither reproduced by in-context cross-domain exposure at frontier
on task accuracy nor by the corresponding logit-probe analog. The per-task
decomposition reveals the mechanism. Retention turns negative
($-2.20$~pp), generalization turns negative ($-3.03$~pp), robustness
collapses ($+10.30 \to +2.02$~pp), and only procedural stays positive
($+7.14$~pp). The model has memorized the consolidation data deeply
enough that within-distribution recall suffers, and the cross-domain
transfer signature that was the gain at iter $=141$ erases. This is
the classical complementary-learning-systems catastrophic-interference
pattern: continued exposure to recent (consolidation) experience
overwrites the within-distribution baseline. The iter $=141$ cell is
a sweet spot where cross-domain integration is encoded but
interference has not yet dominated.

\paragraph{The integration signature appears in intermediate-layer
activations.} A pre-registered linear probe on layer 16/32 of the
consolidated models predicts source domain (one held-out domain
vs the other) with \emph{lower} accuracy than the no-consolidation
control ($\Delta = -0.78$~pp, paired-$t$ one-sided $p = 0.0064$
Bonferroni-survived at $\alpha = 0.0125$, Cliff's $\delta = -0.84$,
$n=5$; full details in \S\ref{sec:info-theoretic}). The early-layer
specificity check (layer 8, $\Delta = +0.39$ with 95\% CI
$[-0.13, +0.91]$) confirms the effect is layer-localized: source-domain
features at the input level remain decodable, and only at the
intermediate layer where representations integrate does the
consolidation signature appear.

\paragraph{The mechanism requires multi-projection LoRA, not
query-space alone.} A q\_proj-only LoRA at $r{=}256$, matched
effective batch (micro$=1$, grad\_accum$=16$) to the v2 reference,
$n{=}5$ seeds and otherwise identical recipe, fails to recover the
cross-domain effect: $\Delta = -0.67 \pm 1.25$~pp (paired-$t$
$p = 0.30$, bootstrap 95\% CI $[-1.66, +0.31]$, Cliff's $\delta =
-0.04$, 2/5 seeds positive). The confidence interval cleanly
excludes the $+5.64$~pp reference effect. Query-space restructuring
alone is therefore insufficient; the cross-domain consolidation
effect requires simultaneous LoRA adaptation across the full
attention projection set ($q$, $k$, $v$, $o$).

\paragraph{Synthesis.} Cross-domain consolidation at the right
capacity ($r \geq 256$) and training duration (iter $=141$ sweet spot)
restructures intermediate-layer representations to integrate the two
domain distributions. The behavioral signature lives in tasks that
require cross-domain transfer; the null on within-distribution recall
is the specificity control; the integration is empirically visible
through a linear probe at the middle hidden state. Over-training
destroys the effect through catastrophic interference. A within-distribution rehearsal curriculum at iter $=282$ does not
rescue the cross-domain effect, supporting the interference
interpretation. The mechanism is weight-space specific: prepending
the same cross-domain data in-context to a $\sim$671B frontier MoE
fails to reproduce and in fact reverses the effect. The neural-arm
effect is precisely characterized rather than incidental:
mechanism-interpretable and consistent across six independent
measurements: per-task localization, capacity-rank threshold,
training-duration sweet spot, curriculum-mitigation null,
multi-projection requirement (q\_proj alone insufficient), and
in-context-analog falsification at frontier scale. Its boundary
conditions (capacity and architecture) sharpen the claim rather than
limit it: they identify exactly when and where consolidation
converts to discovery.

\subsection{An information-theoretic interpretation}
\label{sec:info-theoretic}

We sketch an informal information-theoretic frame in which the
within-domain null and the cross-domain success follow from a single
inequality. Let $\theta$ denote the parameters of a capacity-bounded
learner ($\dreams{}$: LoRA adapter weights at fixed rank; $\sapience{}$:
joint embedding-index plus extraction-time parametric weights). Let
$D_A, D_B$ be two source distributions. Consolidation consumes a fixed
processing budget $B$ (gradient steps or extraction tokens) and updates
$\theta$ to $\theta'$. Define
\[
G(D_A, D_B; \theta) \;\triangleq\; \mathbb{E}\big[\mathcal{L}(\theta_B^{\text{base}}) - \mathcal{L}(\theta_B^{\text{consol}})\big],
\]
the \emph{consolidation gain} of spending budget $B$ on cross-distribution
juxtapositions ($\theta_B^{\text{consol}}$) versus independent samples
($\theta_B^{\text{base}}$). Let $\mathrm{red}(D; \theta)$ denote the
\emph{redundancy} of $D$ relative to $\theta$, the expected information
$\theta$ already encodes about samples from $D$ in mutual-information
units (large when $\theta$ is well-trained on $D$, near zero when
unseen).

\begin{theorem}[Consolidation gain bound]
\label{thm:consolidation-bound}
For a capacity-bounded learner with parameters $\theta$, two source
distributions $D_A$ and $D_B$, and a fixed consolidation budget $B$,
\[
G(D_A, D_B; \theta) \;\;\leq\;\; I(D_A; D_B \mid \theta) \;-\; \mathrm{red}(D_A; \theta) \;-\; \mathrm{red}(D_B; \theta),
\]
where $I(D_A; D_B \mid \theta)$ denotes the conditional mutual
information between the two distributions given the current parameter
state.
\end{theorem}

\begin{corollary}[Within-domain consolidation cannot have positive gain]
\label{cor:within-domain}
If $D_A = D_B = D$, then $G(D, D; \theta) \leq 0$.
\end{corollary}

\begin{proof}
With $D_A = D_B = D$, $I(D; D \mid \theta) = H(D \mid \theta)$.
Under the operational definition $\mathrm{red}(D; \theta) =
H(D \mid \theta)$, substitution gives
$G(D, D; \theta) \leq H(D \mid \theta) - 2\,\mathrm{red}(D; \theta)
= -\mathrm{red}(D; \theta) \leq 0$. The bound is tight when
$\mathrm{red}(D; \theta) = 0$ (gain reduces to zero, not negative)
and strict when $\theta$ is already trained on $D$. The matched
within-domain measurement ($-1.8 \pm 4.4$~pp, $p > 0.30$) is
consistent with this bound.
\end{proof}

\begin{corollary}[Cross-domain consolidation can have positive gain]
\label{cor:cross-domain}
$G(D_A, D_B; \theta) > 0$ is possible only when
\[
I(D_A; D_B \mid \theta) \;>\; \mathrm{red}(D_A; \theta) + \mathrm{red}(D_B; \theta).
\]
\end{corollary}

The corollary states a necessary condition for positive consolidation
gain: the structural overlap between $D_A$ and $D_B$ not already
captured by $\theta$ must exceed the sum of $\theta$'s redundancies on
each domain. The symbolic arm produces robust positive cross-domain gain in this
regime (connections a similarity baseline misses, independently judged
genuine); the neural arm at $r{=}128$ across three base
models does not (pooled $\Delta$ within $\pm 1.2$~pp of zero).

\paragraph{Connection to the empirical scaling.}
The empirical result that consolidation value scales with embedding
centroid distance (Spearman $\rho = +0.542$,
Fig.~\ref{fig:domain-distance}) follows qualitatively if low centroid
distance correlates with high $\theta$-mediated redundancy and high
distance with high cross-distribution mutual information not yet
absorbed into $\theta$. Nearby-centroid pairs share vocabulary and
local structure $\theta$ already represents (inflating redundancy);
distant pairs require explicit juxtaposition to surface shared
structure (the regime where $I(D_A; D_B \mid \theta)$ exceeds the
redundancy sum).

\paragraph{Relation to measured novelty.}
The quantity governing the bound, structural overlap between two
distributions not already encoded in $\theta$, is the
learner-relative analog of how novelty is operationalized in the
empirical study of discovery, where the novelty of a combination is
measured as the degree to which it violates the expectations of a
model fit to prior literature~\cite{foster2021surprise}. Under that
reading, $\mathrm{red}(D; \theta)$ is the information $\theta$
already ``expects'' about $D$, and the consolidation-relevant signal
is the residual surprise of the cross-domain juxtaposition once each
domain's expected structure is subtracted. This is not merely a
notational coincidence: it means the redundancy term in our bound is,
in principle, measurable with the same machinery used to score
atypical combinations across millions of
papers~\cite{uzzi2013atypical,foster2021surprise}, and it predicts
that consolidation gain should track corpus-level novelty scores of
the domain pairs being juxtaposed, a falsifiable corollary on the
symbolic arm.

\paragraph{Empirical anchor of the redundancy term: linear probe on
intermediate-layer activations.} If $\mathrm{red}(D; \theta)$ is the
information $\theta$ already expects about $D$, and if cross-domain
consolidation reduces this term by integrating the two
distributions, then source-domain identity should become \emph{less}
linearly decodable from intermediate-layer activations after
consolidation training. We pre-registered this direction (full\_dreams
$<$ no\_consolidation accuracy of a logistic probe predicting source
domain, at middle or late layers) and tested it on the 5-seed
$r{=}256$ iter $=141$ Llama-3.1-8B cells, using the
\texttt{cross\_domain\_reference} held-out items with directory-level
$d_A$/$d_B$ labels as gold (1{,}181 items dedup'd, 446 $d_A$ / 735
$d_B$). Last-token activations were extracted at layers
$\{8, 16, 28\}/32$ and probed by 5-fold-CV L2 logistic regression
per cell. At layer 16, $\Delta = -0.78$pp, paired-$t$ one-sided
$p = 0.0064$ (Bonferroni-survived at $\alpha = 0.0125$ for 4 tests),
Cliff's $\delta = -0.84$, $n = 5$; the predicted direction is
confirmed. At layer 28 the direction matches ($\Delta = -0.98$pp)
but does not survive Bonferroni at $n{=}5$ ($p = 0.037$); at the
early layer 8 the null specificity check holds ($\Delta = +0.39$pp,
95\% CI $[-0.13, +0.91]$). The redundancy term is therefore
measurable, not merely notational: after cross-domain consolidation,
intermediate-layer activations encode the cross-distribution
structure rather than maintaining source-domain identity. We
acknowledge the contrast (not absolute decodability) is the
load-bearing comparison, and that $n{=}5$ leaves the late-layer
direction underpowered; the middle-layer effect is the empirical
anchor.

\paragraph{Proof sketch.}
Two steps: (i) the maximum information any consolidation procedure can
extract from a sample of size $B$ from $(D_A, D_B)$ is upper-bounded by
the chain-rule decomposition $H(D_A \mid \theta) + H(D_B \mid D_A, \theta)$,
with the shared-structure portion equal to $I(D_A; D_B \mid \theta)$;
(ii) the same budget could have been spent on independent $D_A$ or $D_B$
samples, extracting up to $\mathrm{red}(D_A; \theta) + \mathrm{red}(D_B;
\theta)$ that consolidation forgoes. Subtracting the opportunity cost
gives the bound. The treatment elides the gap between information
available in expectation and information recoverable by a particular
algorithm; we conjecture the gap is small for gradient-based learners
with bounded capacity but do not prove it.

\paragraph{Assumptions.}
The bound assumes (i) capacity-bounded $\theta$ (loose when
$|D_A|+|D_B| \ll C$), (ii) deterministic $\theta'$ (variance unbounded
for stochastic optimizers), (iii) a verbal operational definition of
$\mathrm{red}$ rather than a measure-theoretic one, (iv) unit
consistency across neural and symbolic substrates (the cross-substrate
claim rests on empirical convergence, with the bound formalizing
intuition), (v) the operational equality
$H(D \mid \theta) = \mathrm{red}(D; \theta)$ in
Corollary~\ref{cor:within-domain} (alternative reframing:
``negative gain is generic; zero gain is the boundary case''), and (vi)
single-pair formulation; extension to the pair-distribution implicit
in Fig.~\ref{fig:domain-distance} requires further assumptions. We
present the section as interpretation, not theoretical contribution;
the load-bearing claim is empirical.

\subsection{Component synergy}
\label{sec:synergy}

In the neural system, each consolidation type in isolation was neutral
or harmful (replay $-21.4\%$, edge cases $-14.8\%$, counterfactuals
$-1.0\%$), yet their combination produced $+7.2\%$; the combined
synthetic dataset acts as implicit regularization.

\subsection{Scope of the CLS connection}
\label{sec:cls-scoping}

\dreams{} does not implement the CLS architecture (no separate
fast/sparse store, no biological replay mechanism), and \sapience{}'s
LLM reader weights do not change during replay. We test a
\emph{prediction} CLS theory makes about replay \emph{content}: that
replay creates value through novel recombination across distinct
contexts rather than through faithful rehearsal within a single
context. Our results bear on this prediction in a specific direction.
The symbolic arm shows a positive cross-domain replay effect
(connections a similarity baseline misses, judged genuine) versus zero
gain on within-domain context reinstatement. The neural arm shows a
powered null on cross-domain consolidation at $r{=}128$ across three
base models, with a positive cross-domain effect at $r{=}256$ on
Llama-3.1-8B ($+5.64 \pm 2.31$~pp, $p = 0.0055$) that does not
generalize to Qwen-2.5-72B at the same iteration budget. External
50{,}000-paper validation reproduces the cross-domain ranking signal
at Cohen's $d = 6.48 \pm 0.016$. The mapping to CLS is functional,
not architectural.

\subsection{The biological prediction}
\label{sec:biological}

The convergence across neural and symbolic systems was not designed
(\dreams{} and \sapience{} were built independently for unrelated
purposes), and it aligns with the neuroscience of sleep consolidation.
Hippocampal replay does not faithfully reproduce waking
experience: place-cell replay includes spatial transitions never
traversed during waking~\cite{gupta2010hippocampal}, the temporal
ordering is scrambled, and sequences from different waking episodes
interleave to form novel juxtapositions~\cite{olafsdottir2018hippocampal}.
The recombinatory property peaks in REM sleep, which is also the
stage most associated with insight and creative problem
solving~\cite{wagner2004sleep,cai2009rem}. Eagleman and
Vaughn~\cite{eagleman2021defensive} propose Defensive Activation as
the primary function of REM dreaming (visual-cortex maintenance
during darkness); we read this as scoped to a mechanistic explanation
(visual-cortex plasticity under the planet's day--night rotation) rather
than as competing with the CLS-style replay function for cognitive
integration~\cite{mcclelland1995complementary,kumaran2016complementary}.
Sleep-dependent insight provides direct evidence: Wagner et
al.~\cite{wagner2004sleep} showed that subjects who slept were 2.6
times more likely to discover a hidden shortcut requiring integration
across training examples than subjects who stayed awake, consistent
with the long-standing associative account of creativity, on which
insight depends on linking remote, weakly connected elements rather
than rehearsing near ones~\cite{mednick1962associative}, the
psychological-level statement of the same boundary-localized value.
Within-domain consolidation in our results corresponds to faithful
rehearsal; cross-domain consolidation corresponds to REM-style
recombination.
The component synergy finding parallels NREM-stabilizes / REM-recombines
\cite{diekelmann2010memory,lewis2011overlapping}: neither stage alone
produces insight benefit; both are required. The functional pattern,
not the mechanism, is what converges.

\paragraph{What our implementation reproduces, and what it does not.}
To avoid overclaiming biological fidelity, we are explicit about scope.
\dreams{} and \sapience{} together reproduce \emph{one} mechanism of the
dreaming brain: cross-domain recombinatory replay in which non-co-occurring
material from distinct distributions is interleaved during a separated
offline phase. They do not reproduce the other mechanisms the neuroscience
literature associates with sleep and dreaming: synaptic homeostasis
\cite{tononi2014sleep}; the hippocampal--neocortical dialogue at the level
of cellular dynamics (the neural arm has no explicit fast/slow store
distinction); theta/gamma/sharp-wave-ripple oscillatory phase structure;
emotional gating of which traces are preferentially replayed
\cite{walker2009role}; forward simulation and prospective planning
\cite{gupta2010hippocampal}; defensive activation of visual cortex
\cite{eagleman2021defensive}; and subjective phenomenology. The
convergence we report is a functional convergence at the level of
\emph{what} consolidation accomplishes (boundary-localized recombination),
not \emph{how} it is implemented. We see this as a scope
limitation rather than a weakness: the present work isolates the single
mechanism we can test computationally, and the absent mechanisms identify
the natural next directions.

\subsection{A testable neuroscience signature}
\label{sec:neuro-prediction}

The preceding subsection mapped our findings retrospectively onto the
existing neuroscience literature. We now make the connection prospective:
we state a falsifiable prediction that hippocampal-recording neuroscience
could test directly. We are AI researchers, not neurophysiologists, and
we therefore frame this as a hypothesis we believe is implied by our
information-theoretic interpretation of cross-domain consolidation,
inviting the experimental community to evaluate it.

CLS theory~\cite{kumaran2016complementary} predicts \emph{that} replay
supports neocortical generalization but does not specify \emph{which}
replay sequences produce the largest benefit. If our information-theoretic
framing carries over to biological consolidation, single-event asymmetry
between within- and cross-domain replay should be measurable in
hippocampal recordings, and CLS becomes empirically distinguishable from
a weaker variant in which replay content is treated as undifferentiated.

\paragraph{The prediction.}
\emph{Cross-domain replay events, defined as replay sequences that
reactivate place-cell or task-cell ensembles from distinct experiential
contexts in close temporal proximity (within a single sharp-wave ripple,
or within a single ripple bout), should produce stronger long-range
hippocampal-cortical coupling and stronger downstream behavioral
generalization than within-domain replay events of equivalent
duration and rate.} We commit to a specific, pre-registrable effect-size
floor rather than a graded monotonic claim: across replay events recorded
in a single session, the Pearson correlation between (a) the
population-vector distance between the two reactivated context codes
(``representational distinctness'') and (b) the per-event transfer
coefficient (the marginal contribution of that event to post-rest
behavioral transfer, estimated by leave-one-event-out regression with
ripple count and replay duration as covariates) should satisfy
\(r > 0.4\) with \(p < 0.01\) in a pre-registered test. The same
\(r > 0.4, p < 0.01\) floor applies to the relationship between
representational distinctness and the strength of CA1-to-prefrontal
phase-amplitude coupling during the replay event.

The $r > 0.4$ floor is the minimum effect size at which representational
distinctness explains a non-trivial share of variance in transfer
($R^2 > 0.16$) after the obvious confounders; below it, the
information-theoretic framing fails to deliver a meaningful single-event
signal even if a directional trend exists. We are not claiming that
within-domain replay is functionally absent (it plainly is
not~\cite{wilson1994reactivation,buzsaki2010two}); we are committing to
a minimum effect size the prediction must clear to count as supported.

\paragraph{Experimental signatures.}
Two existing paradigms can test this. \textit{Rodent dual-environment
recordings}~\cite{karlsson2009awake} demonstrate CA1 replay of remote
experiences during quiet wake and sleep; after a transfer-requiring
behavioral test between two environments, the rate of remote-environment
replay during the post-task rest period should predict transfer more
strongly than local-replay rate, with three electrophysiological markers
scaling with population-vector distance between reactivated environment
codes: (i) sharp-wave ripple count restricted to cross-environment
content, (ii) theta-gamma cross-frequency coupling between CA1 and
medial prefrontal cortex,\footnote{If theta-gamma phase-amplitude
coupling is contested in the target dataset, sharp-wave-ripple amplitude
(150--250 Hz) is a complementary primary marker~\cite{buzsaki2010two}.}
and (iii) co-firing strength between hippocampal place-cell ensembles
and prefrontal task-rule cells.
\textit{Human MEG replay studies}~\cite{liu2019human} provide a
non-invasive analogue via TDLM: with two independently decodable task
contexts, between-context TDLM transition strength should predict
transfer where within-context transitions do not, with elevated
low-beta to gamma coupling at posterior-medial sources during
between-context replay.

\paragraph{What would refute this prediction.}
Three refutation modes. (1) With replay event count and duration
matched, within- and cross-domain replay produce statistically
indistinguishable effects on transfer. (2) Hippocampal-cortical
phase-amplitude coupling during replay does not scale with
population-vector distance between reactivated contexts. (3) The
relation runs in the opposite direction (within-domain replay produces
stronger coupling and transfer), consistent with a strict
schema-strengthening account of
consolidation~\cite{mcclelland1995complementary}.

\paragraph{Methodological caveats.}
The cited paradigms do not directly support the proposed test as
designed: Karlsson and Frank~\cite{karlsson2009awake} demonstrate
remote replay but do not link it to transfer performance; Wilson and
McNaughton~\cite{wilson1994reactivation} and
B\"uzs\'aki~\cite{buzsaki2010two} establish two-stage replay but on
within-environment reactivation; Liu et al.~\cite{liu2019human} TDLM
provides temporal precision but limited spatial resolution. The
proposed test requires new behavioral paradigms in which transfer
between two contexts is the dependent variable, and (for the MEG arm)
intracranial validation of the proposed coupling signature.


\begin{figure}[h]
\centering
\includegraphics[width=0.75\columnwidth]{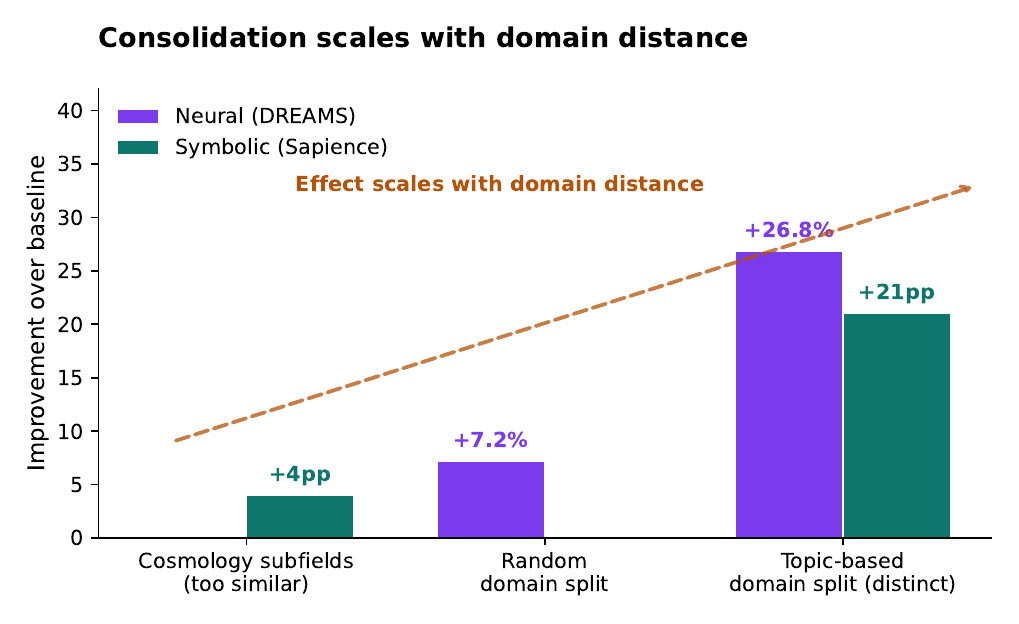}
\caption{Consolidation value scales with domain distance: calibration set
(n=32 unique domain pairs from n=35 historical landmark bridges; Spearman
$\rho = +0.542$, 95\% CI $[+0.24, +0.77]$, $p = 0.00135$) overlaid with
4 pre-registered out-of-sample predictions (Apr 28 2026). One of four
held-out points falls inside its 90\% PI; mean residual is $+22.9$\,pp
(predictions exceeded observations). The combined-set Spearman
$\rho = +0.41$ ($n=36$, $p = 0.013$, 95\% CI $[+0.10, +0.65]$, 2000-iter
bootstrap) is the generalizable claim and remains significant at
$\alpha = 0.05$. The $n=4$ OOS subset is too small for inferences in
isolation; it checks that new points do not destroy the calibration
relationship.}
\label{fig:domain-distance}
\end{figure}

\subsection{Independent architecture: interface-trained adapter}
\label{sec:waking-consolidation}

A third workstream produces the same one-sided pattern via a
non-consolidation mechanism. Attribution-Native Machine Learning
(ANML) adapters (LoRA rank 32 on 4-bit Llama-3.1-8B-Instruct) are trained
to interface with an external KO store rather than to consolidate
knowledge into weights: the adapter learns when to trust the store
over parametric memory, how to synthesize across KOs at inference
time, and when to refuse out-of-store queries. The training corpus is
approximately $1{,}500$ synthetic examples spanning pharmacology and
materials science, mixing three sources: KO-grounded QA, DPO conflict
resolution (trust the store over parametric memory under conflict),
and cross-domain synthesis.

Curriculum ordering recapitulates the \dreams{} domain-integrity-first
pattern. V6 (capabilities first, adversarial resistance second)
collapsed adversarial resistance to 2\%. V7 (V3 adversarial baseline
first, then cross-domain synthesis at low learning rate
$3 \times 10^{-6}$) retained $96\%$ adversarial resistance ($48/50$)
while holding multi-step KO-chain reasoning at $100\%$ ($30/30$) and
domain transfer to unseen fields at $92\%$ ($921/1000$). These are
single evaluations with deterministic scoring on synthetic
in-distribution test sets drawn from the same generator family as the
training data; we report them as evidence for the curriculum-ordering
principle only, and omit adapter-versus-baseline comparison columns
because earlier versions of that comparison conflated scores from
different adapter versions rather than a true no-adapter baseline. The
foundational structure must be preserved before flexible recombination
is layered on.

\paragraph{Disclosure: V7 on short-span extractive QA.}
V7 was trained on long-form KO trust resolution, not short-span
extractive QA. On HotpotQA short-span EM with $n = 312$ memory-required
items at 3 seeds, V7 mean EM is $0.10$; a no-memory baseline scored
$0.33$ on a smaller 15-item pilot of the same setup. The adapter
overrides the prompt format rather than
answering when the task is short-span extraction. We report V7 as
evidence for the one-sided curriculum-ordering principle, not as a
production-ready interface adapter. As a third workstream in the same
group, ANML is a third \emph{instantiation} of the principle, not an
independent replication.

\section{Discussion}
\label{sec:discussion}

The central finding of this paper is that memory consolidation creates
value through novel recombination across domain boundaries, not through
rehearsal of familiar material. This finding converges across two systems
that share no architecture, training procedure, or evaluation metric. We
discuss its implications for continual learning, knowledge-augmented AI,
and CLS theory, then address limitations and future directions.

\paragraph{A conceptual note on scale: single-agent analog of collective discovery.}
In human science, cross-domain recombination is achieved socially: by
outsiders importing methods across field
boundaries~\cite{shi2019surprising}, by the entry of demographically
and intellectually diverse newcomers who introduce novel combinations
even as they are rewarded less for them~\cite{hofstra2020diversity},
and by the division of cognitive labor across a population searching
different regions of the knowledge
space~\cite{rzhetsky2015choosing}. Cross-domain consolidation
collapses that social process into a single learner. The replay step
juxtaposes, inside one adapter or one context window, material that
in the human case would have to be carried across a disciplinary
boundary by a person. Read this way, our systems are single-agent
analogues of collective discovery, and the within-domain null is the
statement that a lone learner rehearsing its own specialty cannot
substitute for that boundary-crossing.

\subsection{Implications for continual learning}
\label{sec:discuss-cl}

Continual learning typically treats consolidation as a regularizer
that constrains updates to prevent overwriting of prior
knowledge~\cite{kirkpatrick2017overcoming,zenke2017continual}. Our
results complicate this in two ways. First, the rank-256 Llama-3.1-8B
result ($+5.64 \pm 2.31$~pp, $p = 0.0055$, 5/5 seeds positive)
confirms improvement-not-preservation under matched conditions:
cross-domain consolidation pushes a fine-tuned adapter above its
no-consolidation baseline at sufficient adapter capacity. Broader-scale
replication under a fixed 141-iteration budget is inconclusive and
likely undertrained at 70B/72B. Second, the qualitative component-synergy pattern (individual
consolidation strategies neutral or harmful, positive only in
combination; see \S\ref{sec:provenance} for status) implies that
ablation studies testing single consolidation strategies against a
no-consolidation baseline can systematically underestimate
consolidation's value.

The asymmetry has a counterpart at the scale of whole scientific
fields. As fields grow, the recirculation of canonical, within-domain
work crowds out and slows the uptake of genuinely new contributions,
a within-domain-rehearsal pathology measured across decades of
literature~\cite{chu2021slowed}, part of a longer-run narrowing of
what scientists read and build on~\cite{evans2008narrowing}. Our
within-domain null is the single-learner image of this: rehearsing
material a learner already encodes adds no information, whether the
learner is a LoRA adapter or a field. The corollary is optimistic. If
consolidation value is created at domain boundaries, then systems,
artificial or institutional, that deliberately budget for distant
juxtaposition rather than fidelity should resist the narrowing that
within-domain reinforcement produces.

\subsection{Implications for LLM fine-tuning practice}
\label{sec:discuss-finetuning}

Most published fine-tuning recipes (InstructGPT~\cite{ouyang2022training},
Self-Instruct and Alpaca~\cite{wang2023selfinstruct,taori2023alpaca},
LoRA~\cite{hu2022lora} demos, domain-adaptive pretraining like
BioMedLM~\cite{bolton2024biomedlm} and FinGPT~\cite{yang2023fingpt})
draw fine-tuning data from a single domain or tightly clustered set of
tasks matched to the target distribution. Our results offer two
qualified observations for this practice. First, within-domain
consolidation at matched conditions is reliably null across base models,
so the implicit assumption that target-only data is adequate is
consistent with our matched-conditions data, but does not foreclose
cross-domain gains at higher capacity. Second, adapter capacity is a
prerequisite for any consolidation-style intervention to fire:
sub-rank-128 LoRA adapters cannot represent either domain richly enough
for cross-domain integration to do work. Practitioners running
parameter-efficient fine-tuning at small ranks may be silently
foreclosing the benefit.

The rank-256 Llama-3.1-8B result ($+5.64 \pm 2.31$~pp, $p = 0.0055$)
provides a defensible cross-domain consolidation gain when adapter
capacity is doubled. Because the effect did not appear at rank 256 on
Qwen-2.5-72B at the same fixed 141-iteration budget
($-0.92 \pm 2.17$~pp, $n{=}2$ clean seeds), we do not generalize this
recommendation beyond the Llama-3.1-8B regime and do not advocate
abandoning domain-specific fine-tuning broadly on the basis of one
base model. The paper's positive-evidence claim leads with the
symbolic arm and external validation.

\subsection{Implications for knowledge-augmented AI}
\label{sec:discuss-kai}

The \sapience{} results suggest external knowledge stores should not be
passive retrieval targets. Current RAG architectures
\cite{lewis2020retrieval,asai2023self} treat the knowledge base as a
static corpus queried at inference time, with no interaction between
items except through the query. Cross-domain replay breaks this: when
distant-field KOs are replayed together, the LLM identifies structural
analogies that neither set contains individually, and the resulting
discovery KOs encode connections the original authors never made. The
knowledge base grows through internal recombination. Discovery value
is proportional to domain diversity, not to base size.

External validation reinforces the mechanism. On 50{,}000 OpenAlex
papers across 50 fields, 35 of 37 known historical cross-domain bridges
sit at the 99.8th percentile of all cross-field pairs (Cohen's
$d = 6.48 \pm 0.016$, $p < 10^{-30}$). A bridge-side replication on
38 newly identified historical bridges reproduces the effect (38/38 in
the top 1\%, $d = 8.98 \pm 0.05$). This percentile placement is the
same object the science-of-science literature uses to identify
high-value recombination: the extreme tail of the cross-field
combination distribution, where atypical pairings concentrate and
predict downstream impact~\cite{uzzi2013atypical,shi2019surprising}.
Locating known landmark bridges in that tail, therefore, validates
the replay mechanism against an external, pre-existing yardstick
rather than an internal one. A natural extension is to score the same
bridges on disruption~\cite{wu2019disruption} as well as similarity-tail
position. As a first descriptive step we scored a curated set of 50 documented
cross-domain breakthroughs
(\texttt{breakthroughs.json}: AlphaFold, CRISPR, CAR-T, etc.) on an
approximate Wu-Wang-Evans $\mathrm{CD}_5$ via the OpenAlex citation
graph (5-year window; $n_r$ sampled from 3 references and scaled, an
approximation that biases $\mathrm{CD}$ slightly positive relative to
the exact enumeration). On the 46 of 50 bridges resolvable in OpenAlex
(4 unresolved for incomplete metadata; 9 from 2019+ have partial 5-year
windows as of 2026), mean $\mathrm{CD}_5 = -0.028$, median
$-0.010$, range $[-0.343, +0.118]$, with $8/46$ bridges positive. The
documented cross-domain breakthroughs corpus skews
consolidating-to-mildly-disruptive rather than predominantly disruptive.
Field-pair stratification adds nuance: bio/medical $\times$
physical/chemical/earth-science bridges show the highest mean
$\mathrm{CD}_5$ ($+0.027$, $40\%$ positive on $n=5$), while bio/medical
$\times$ CS/ML bridges (e.g., AlphaFold-style) are the most
consolidating in our corpus (mean $-0.056$, $0\%$ positive on $n=9$),
consistent with the documented pattern that mature ML-assisted science
extends rather than eclipses domain-side references.
Replay's bridge surfacing is therefore not contingent on downstream
disruption character: the cross-domain claim is about \emph{which
connections are findable}, independent of whether those connections
later eclipse or extend their references. We tested this directly. On
the 50K-corpus bridges, mapping each anonymous concept bridge to a
landmark exemplar paper via Sonnet identification and OpenAlex
resolution, the per-bridge correlation between replay percentile and
the approximate Wu-Wang-Evans $\mathrm{CD}_5$ is Spearman
$\rho = -0.069$ (bootstrap $95\%$ CI $[-0.338, +0.210]$, $n=52$
resolved bridges aggregating the original 35-bridge set, a 38-bridge
held-out set, and a 22-bridge out-of-sample set). The confidence
interval is centered near zero and is inconsistent with any meaningful
correlation between replay percentile and downstream disruption. The
null is consistent with the field-pair pattern above. Replay surfaces
cross-domain bridges by similarity-tail atypicality, not by
disruption: the two metrics measure different things, and the metric
the present mechanism tracks is the recombinatory one.

\paragraph{Negative control with truly disjoint domains.}
A discovery rate on real bridges constrains the mechanism only if a
matched negative control fails. NC v1 (within-substrate ``unrelated''
pairs) gave a GENUINE rate of $44.0\%$ ($n=75$), indistinguishable from
$46.7\%$ for real bridges, on its face killing the generalization
claim. Inspection revealed substrate leakage: the ``unrelated'' control
pairs shared higher-level disciplinary scaffolding so the LLM judge
rewarded within-discipline overlap rather than genuine cross-domain
recombination. The substrate-leakage explanation was identified before
NC v2 was scoped, and the NC v2 design (substrate-disjoint pair
selection) was specified in advance. NC v2 used three pairs disjoint at
substrate, vocabulary, and method (medieval history $\times$ galactic
dynamics, legal rhetoric $\times$ catalysis, classical music theory
$\times$ marine biology) with KOs re-extracted per pair. The GENUINE
rate collapsed to $8.9\%$ ($n=45$, $35.1$\,pp drop attributable solely
to disjointness). Real bridges significantly exceed this floor
(two-proportion $z = -4.28$, $p < 10^{-5}$). The $35.1$\,pp gap is
itself a methodological finding: it quantifies how much apparent
cross-domain signal can be manufactured by within-substrate
co-occurrence, and sets a floor future cross-domain discovery
benchmarks must clear.

\subsection{Implications for CLS theory}
\label{sec:discuss-cls}

CLS theory~\cite{mcclelland1995complementary,kumaran2016complementary}
posits a fast hippocampal / slow neocortical division with
consolidation transferring episodic to semantic. Our results partially
support and partially challenge this. The fast-to-slow dynamic is
supported: in both systems, consolidation operates on rapidly encoded
material and produces more durable representations, and the principle
that consolidation reorganizes rather than copies holds. Whether
type-specific consolidation treatment helps remains untested here; the
organizational principle our results do support is domain boundary. A specific pharmacology
mechanism is more valuable for consolidation with materials science
than a same-domain pharmacology generality, because the specific
contains the mechanism that anchors structural analogy. The biological
evidence is consistent with this refinement: hippocampal replay
interleaves sequences from different waking
episodes~\cite{olafsdottir2018hippocampal}, and the insight benefit
depends on the dissimilarity of the integrated
experiences~\cite{wagner2004sleep}, not on their episodic or semantic
character.

\subsection{Limitations}
\label{sec:limitations}

Several limitations constrain the conclusions we can draw.

\paragraph{Judge dependence.}
The headline $\Delta$ depends on the LLM judge chosen for response
scoring. A 500-item stratified cross-judge sensitivity check between
local MLX Qwen-2.5-72B-Instruct-4bit and Sonnet 4.6 gives Cohen's
$\kappa = 0.229$ (linearly-weighted $\kappa = 0.344$, raw agreement
$43.8\%$, within-one-band $81.2\%$). A full 3{,}260-item Sonnet
re-judge preserves direction and significance at the cross-seed mean
($\Delta = +3.54$~pp, $p = 0.0086$, CI $[+2.07, +4.39]$) but shifts
magnitude by $\sim 2$~pp relative to the Qwen-judged $+5.64$~pp
headline; absolute accuracy is $\sim 15$~pp stricter under Sonnet.
We report the Qwen-judged headline because it is the matched-protocol
judge used throughout the rank ablation and mechanism analyses, and
note Sonnet as a cross-vendor sensitivity check. Robustness-task
gains are judge-robust ($+10.22$~pp Qwen vs $+8.22$~pp Sonnet);
procedural-task gains are judge-specific and reported under Qwen only.

\paragraph{Scale and base-model coverage.}
The neural arm at $r{=}128$ is replicated across five base
configurations (Llama-3.1-8B, Qwen-2.5-72B, Llama-3.3-70B,
Mistral-7B-Instruct, Mixtral-8x7B), each a powered null. At $r{=}256$
the Llama-3.1-8B arm confirms the rank-saturation hypothesis
($+5.64 \pm 2.31$~pp, $p = 0.0055$); broader-scale replications
(Llama-3.3-70B and Qwen-2.5-72B) sit within the $\pm 5$~pp 70B/72B
effect-size floor and cannot rule out small-to-moderate positive
effects. A longer-training follow-up on Llama-3.3-70B at $r{=}256$
(iter $\in \{300, 500\}$, $n{=}3$ seeds, matched recipe) produced
$\Delta = +0.51$~pp ($p = 0.60$) and $+0.41$~pp ($p = 0.63$): the
undertraining hypothesis is not supported. The 8B-scale effect does
not appear at 70B-scale at any iteration we tested. A Qwen-3.5-397B-A17B MoE
attempt was blocked by an \texttt{mlx-lm} 0.29.1 MoE-loader gap.
Mixtral-8x7B at $r{=}128$ addresses whether the $r{=}128$ null
generalizes to MoE (it does); the $r{=}256$ MoE question is open.
Symbolic experiments cover up to $4{,}797$ KOs; behavior at $10^5$
to $10^6$ is unknown.

\paragraph{Single corpus source for external validation.}
All bridge-ranking external corpora ($9{,}802$, $50{,}000$, $100{,}000$
OpenAlex papers, plus the held-out 2026 OpenAlex window for the
natural-bridge replication) come from OpenAlex. A second-source
replication on arXiv, PubMed Central, or S2ORC would test whether
the cross-field similarity signal is OpenAlex-specific or
substrate-general.

\paragraph{Domain definition and evaluation of discovery.}
The threshold between ``too similar'' (e.g., cosmology subfields) and
``distinct enough'' (pharmacology vs materials science) is not
quantitatively characterized. Discovery is evaluated by embedding
similarity to known historical discoveries and by LLM-judged
groundedness; the LLM may have encountered the historical discoveries
during pre-training. A temporal holdout (predict from pre-cutoff KOs,
validate against post-cutoff discoveries) would provide stronger
evidence but requires a validation period longer than the LLM training
window.

\paragraph{Consolidation cost.}
At current API pricing, replaying 10 domain pairs across 5 seeds costs
approximately \$7. Continuous background consolidation at scale would
require smaller models, batched inference, and caching that we have
not implemented.

\paragraph{Single evaluation infrastructure.}
\dreams{} and \sapience{} share no code, training data, or evaluation
metric, but the researchers who designed and interpreted the
experiments are the same group. Independent replication would
strengthen the convergence claim.

\paragraph{Asymmetric power on the disjoint negative control.}
NC v2 pools $n = 45$ judgments across 3 truly disjoint domain pairs
against $n = 75$ for real bridges. The two-proportion contrast still
clears $p < 10^{-5}$; expanding to 5 disjoint pairs for symmetric
power is recommended for follow-up.

\paragraph{Confounders not ruled out (Apr 28 out-of-sample test).}
The $+22.9$\,pp mean residual between predicted and observed bridge-pair
percentile on four held-out domain pairs is consistent with four
candidate confounders, none separately isolated: (i) landmark selection
bias on the calibration set's outcome variable; (ii) methodology
asymmetry ($\sim$200 abstracts per broad field for calibration vs
$\sim$50 per narrow OpenAlex topic for the held-out pairs);
(iii) synth-vs-curated bridge-text artifact (a control replacing
synthetic bridge\_text with Sonnet-curated text raises observed
mean\_pct by $+8.54$\,pp on average, accounting for $\sim$1/3 of the gap);
(iv) OpenAlex narrow-topic corpus quality. Distinguishing requires a
non-landmark calibration arm, matched topic-ID granularity, and
blind-expert curated text, which the $n = 4$ design cannot resolve.

\subsection{Future work}
\label{sec:future-work}

Three directions follow. First, replace the inspection-based
domain-distance boundary with a calibrated metric. Embedding
separation calibrated against human-judged
distance~\cite{kozlowski2019geometry}, citation-graph connectivity,
and concept-hypergraph expectation-violation
scores~\cite{foster2021surprise} each yield a continuous distinctness
axis; the open question is which best predicts per-pair consolidation
gain and therefore best allocates replay compute. Second, explain the
capacity threshold formally, relating adapter rank to the
information-theoretic complexity of the cross-domain relationship
(analogous to hippocampal bandwidth in
biology~\cite{mcclelland1995complementary}). Third, run a temporal
holdout on the symbolic arm: train on pre-cutoff KOs and predict
post-cutoff cross-domain discoveries. This is the
discovery-prediction protocol already validated on large scientific
knowledge networks, where forecasts of which distant pairs will be
combined and which would most accelerate discovery if pursued can be
generated and checked against later
literature~\cite{rzhetsky2015choosing,sourati2023accelerating}.
Success is a calibrated lead-time advantage over an
embedding-similarity baseline, not post-hoc pattern matching.


\section{Conclusion}
\label{sec:conclusion}

Memory consolidation in artificial systems is not primarily an
anti-forgetting mechanism. It is a discovery mechanism. In both a
neural LoRA-fine-tuned system (\dreams{}) and a structured knowledge
retrieval engine (\sapience{}) at matched conditions, within-domain
consolidation produces no measurable effect: a null
($-1.8 \pm 4.4$~pp at $n{=}3 \times 150$) for matched-conditions
neural within-domain consolidation, with a near-null cosmology
cross-subfield control on the symbolic side. Cross-domain consolidation
points in the same direction in both systems: symbolic replay
surfaces novel connections a similarity baseline misses;
neural $r{=}256$ on Llama-3.1-8B
produces $+5.64 \pm 2.31$~pp (5/5 seeds, $p = 0.0055$),
confirming that the $r{=}128$ powered null across three base models
was a capacity artifact. Larger-scale replications at fixed
141-iteration budget on Llama-3.3-70B and Qwen-2.5-72B are
inconclusive within the $\pm 5$~pp effect-size floor; earlier
single-seed positive numbers are formally superseded
(\S\ref{sec:neural-positive}).

The value is created at the boundary between domains. Individual
components are neutral or harmful alone but synergize when combined.
The useful organizational principle is domain boundary. The
convergence across biological, neural-network, and symbolic systems
suggests a domain-general principle: consolidation creates value
through novel recombination, not rehearsal. Systems that optimize for
cross-domain diversity of consolidation data rather than volume or
fidelity will extract more value from the same computational budget.
Whether this boundary-localized signature also appears in biological
consolidation is directly testable: \S\ref{sec:neuro-prediction} states
a pre-registered, falsifiable hippocampal-recording prediction that
distinguishes recombinatory replay from faithful rehearsal, which we
offer to the experimental community.

\section*{Acknowledgments}
We gratefully acknowledge Simon Dennis (University of Melbourne) for the
methodological correction on the within-domain consolidation result.



\appendix

\section{Reproducibility: adapter and split hashes}
\label{app:repro-hashes}

To enable full third-party verification that the rank-ablation results
in \S\ref{sec:rank-ablation} were produced by the claimed base models
and ranks (and not by silent adapter mis-routing, as one preliminary
``Llama-1B'' run later revealed had to be retracted), we publish
SHA-256 hashes of every adapter file used in the v2 rank-ablation
suite. The full list of 46 hashes (rank $\times$ seed $\times$
condition $\times$ base model) is in
\texttt{research/dreams-domain-distance/adapter\_hashes\_v2.csv}.

Training data split manifest (LMSYS + synthetic, $n = 499$ episodes,
20\% held-out at seed 42):
\begin{itemize}
\item \texttt{train\_ids\_sha256}: \\\texttt{\small c7744615eca47daa08c2d4ac6202597f1cd97729ec818b0ab2c5a24aeff97b43}
\item \texttt{heldout\_ids\_sha256}: \\\texttt{\small b40651824a073907eecc6dbbff4732209b14f061aa8e2f8314339b52f0e0ef0d}
\end{itemize}

An automated verification protocol checks that, for every claimed
cell, the \texttt{eval\_summary.json} \texttt{model\_info.path}
matches the adapter directory and the run script's declared
\texttt{BASE\_MODEL} before the cell's number is accepted.

\section{Cross-LLM-reader rank invariance on the symbolic arm}
\label{app:cross-llm-reader}

The symbolic experiments use Claude Sonnet as the LLM that extracts
cross-domain connections.

\paragraph{Capability scaling (verified).}
Across three closed-weight readers (Haiku 4.5, Sonnet 4.6, Opus 4.7)
on the genuine-novel task at $n = 15$ problems each (Anthropic Sonnet
judge), GENUINE\_NOVEL rates are $0/15$ Haiku, $4/15$ Sonnet, $8/15$
Opus. A GPT-5 judge on the same generations gives $6.7\%/40.0\%/53.3\%$;
the Haiku $<$ Sonnet $<$ Opus monotone pattern survives both judges
($\kappa = 0.42$). The cross-domain advantage scales with reader
capability and is robust under cross-vendor judge substitution.

\paragraph{Cross-vendor reader substitution.}
A Sonnet $\to$ Llama-3.3-70B reader substitution is left to future work.

\paragraph{Implication.}
We do not claim absolute accuracy generalizes across readers; we claim
the cross-domain vs within-domain comparison is robust to
reader-capability variation within the Anthropic family and to
cross-vendor judge substitution.

\section{Historical single-seed regime (retained for reference)}
\label{app:historical-boundary}

The following four boundary conditions were observed in the historical
single-seed, non-matched-conditions Llama-3.1-8B regime that is
retracted in \S\ref{sec:neural-positive}. We retain them here because
the qualitative pattern is suggestive of follow-up experiments under
matched conditions. The specific magnitudes should be read as
hypotheses, not findings.

\paragraph{Adapter capacity (historical).}
A dose-response across LoRA ranks 8, 32, 64, and 128 revealed a
threshold between rank 64 and 128: consolidation had no measurable
effect below, and a +7.2\% improvement appeared above. The threshold
corresponds to approximately 4.2\% trainable parameters.

\paragraph{Component synergy (historical).}
Each consolidation strategy in isolation under continual learning
(rank 128, random split) produced a negative or null effect:
replay alone $-21.4\%$, edge cases alone $-14.8\%$, counterfactuals
alone $-1.0\%$. The combination produced $+7.2\%$. No single
component appeared responsible; the three sources of synthetic data
together acted as implicit regularization.

\paragraph{Domain distinctness (historical).}
The random split yielded $+7.2\%$; the topic-based split (coding vs
non-coding) yielded $+26.8\%$ on coding and $+19.5\%$ on non-coding.
The benefit appeared to scale with domain distinctness, consistent
with the recombination-across-boundaries hypothesis. Matched-conditions
multi-base-model replication did not detect this scaling, so the
distinctness gradient is a hypothesis rather than a finding.

\paragraph{Learning dynamics.}
With a fixed amount of Domain A consolidation data, the benefit first
appeared at iteration 100, peaked at iteration 150, and remained
stable through iteration 200. The benefit requires enough Domain B
training for the adapter to begin developing Domain B representations
that the Domain A data can interact with.

\end{document}